\documentclass[11pt,letterpaper]{article}
\usepackage[margin=1in]{geometry}
\usepackage{amsmath,amssymb,amsthm}
\usepackage{booktabs}
\usepackage{array}
\usepackage{tabularx}
\usepackage{hyperref}
\usepackage{enumitem}
\usepackage{graphicx}
\usepackage{xcolor}
\usepackage{multirow}
\usepackage{tikz}
\usepackage{pgfplots}
\pgfplotsset{compat=1.17}

\hypersetup{colorlinks=true,linkcolor=blue,citecolor=blue,urlcolor=blue}

\newtheorem{definition}{Definition}[section]
\newtheorem{proposition}{Proposition}[section]
\newtheorem{theorem}{Theorem}[section]
\newtheorem{corollary}{Corollary}[section]
\newtheorem{thesis}{Thesis}[section]

\title{\textbf{Governing What You Cannot Observe:\\Adaptive Runtime Governance for Autonomous AI Agents}}
\author{Germ\'an Mar\'in\textsuperscript{1}, Jatin Chaudhary\textsuperscript{2}\\[0.5em]
\textsuperscript{1}\texttt{gmarin@cosasocial.com}\\
\textsuperscript{2}Department of Computing, University of Turku, Turku, Finland\\
\textsuperscript{2}\texttt{jatin.chaudhary@utu.fi}}
\date{April 2026}

\begin{document}
\maketitle

\begin{abstract}
Autonomous AI agents can remain fully authorized and still become unsafe as behavior drifts, adversaries adapt, and decision patterns shift without any code change. We propose the \textbf{Informational Viability Principle}: governing an agent reduces to estimating a bound on unobserved risk $\hat{B}(x) = U(x) + SB(x) + RG(x)$ and allowing an action only when its capacity $S(x)$ exceeds $\hat{B}(x)$ by a safety margin. The \textbf{Agent Viability Framework}, grounded in Aubin's viability theory, establishes three properties---monitoring (P1), anticipation (P2), and monotonic restriction (P3)---as individually necessary and collectively sufficient for documented failure modes. \textbf{RiskGate} instantiates the framework with dedicated statistical estimators (KL divergence, segment-vs-rest $z$-tests, sequential pattern matching), a fail-secure monotonic pipeline, and a closed-loop Autopilot formalised as an instance of Aubin's regulation map with kill-switch-as-last-resort; a scalar Viability Index $VI(t) \in [-1,+1]$ with first-order $t^*$ prediction transforms governance from reactive to predictive. Contributions are the theoretical framework, the reference implementation, and analytical coverage against published agent-failure taxonomies; quantitative empirical evaluation is scoped as follow-up work.
\end{abstract}


\section{Introduction}\label{sec:introduction}

Today's agent platforms, AWS, Microsoft, Google, control permissions, tool access, logging, and policy enforcement. That is important, but it is not enough. An agent can remain fully authorized and still become unsafe over time.

The deployment of large language model (LLM) agents in production environments has accelerated dramatically. Unlike traditional software, LLM agents exhibit three properties that make conventional governance insufficient~\cite{gaurav2025}:

\begin{enumerate}[label=(\roman*)]
\item \textbf{Behavioral non-stationarity.} The distribution of an agent's actions can shift silently over time due to model updates, prompt drift, or adversarial manipulation, without any explicit code change.
\item \textbf{Compositional risk.} Individual operations may each appear benign, yet their composition---a sequence of small transfers below a regulatory threshold---constitutes fraud.
\item \textbf{Emergent discrimination.} Adaptive learning systems can amplify initial data imbalances into systematic bias against population segments.
\end{enumerate}

Recent empirical evidence confirms the severity. Shapira et al.~\cite{shapira2026}, and Gaurav et al.
~\cite{gaurav2025} deployed autonomous LLM agents in a live laboratory environment and documented eleven failure case studies over a two-week red-teaming period, including agents entering resource-consuming infinite loops lasting nine days, complying with unauthorized requests that disclosed 124 email records, accepting spoofed identities leading to full system compromise, and enabling cross-agent corruption through shared communication channels. Indirect prompt injection---present in over 73\% of evaluated production deployments~\cite{nassi2025}, can manipulate agents into authorizing operations far beyond their intended scope.

The regulatory pressure is equally concrete. The EU AI Act (2024/1689) imposes fines of up to 35M~EUR or 7\% of global revenue for non-compliant high-risk AI systems. Basel, PSD2, HIPAA, and Solvency~II each demand quantifiable risk management, bias monitoring, and full decision traceability.

Existing approaches fall into three categories with fundamental limitations. Rule-based guardrails cannot detect distributional drift and are trivially evaded by sophisticated adversaries. LLM-based evaluators introduce latency, cost, and the recursive problem of governing the governor. Static policy engines, such as Amazon Bedrock AgentCore's Cedar-based authorization~\cite{agentcore2025}, provide strong per-request guarantees but are fundamentally deterministic: they evaluate each action as a stateless tuple with no memory of prior behavior, no distributional context, and no capacity to detect that the agent's operational patterns are drifting over time.

RiskGate bridges this gap. But more than a system, this paper offers a theoretical reframing of the governance problem itself:

\begin{quote}
\emph{Governing an autonomous agent is equivalent to continuously estimating the bound on what you cannot observe, $\hat{B}(x)$, and acting only when your observed capacity $S(x)$ exceeds it.}
\end{quote}

This principle, the \textbf{Informational Viability Principle}, unifies all statistical estimators, the AVF's three properties, and the Viability Index under a single conceptual structure. It provides a unified motivation for why P1, P2, and P3 arise as the natural governance requirements for any system that explicitly separates $S(x)$ from $\hat{B}(x)$ as estimable quantities.

\paragraph{Position in the governance stack.} RiskGate is complementary to production agent platforms such as Amazon Bedrock AgentCore~\cite{agentcore2025}, which enforce static, identity-centric authorization via Cedar policies~\cite{cedar2024} and answer \emph{``is this specific action allowed right now?''}. RiskGate targets an orthogonal question, \emph{``is $\hat{B}(x)$ growing?''}, that static policy engines do not address by design. The two operate at different layers and are not like-for-like: platforms like AgentCore provide operational maturity, IAM, and scale that a research-stage framework cannot match, while RiskGate provides adaptive $\hat{B}(x)$ estimation and the Viability Index that static policy engines do not target. A plausible integration is to run RiskGate as governance middleware on top of such a platform, which we flag as future work rather than a competitive claim.

\subsection{Contributions}

\begin{enumerate}
\item \textbf{The Informational Viability Principle.} We establish that governing an autonomous agent reduces to estimating $\hat{B}(x) = U(x) + SB(x) + RG(x)$ and enforcing $\hat{\Phi}(x) = S(x) - \hat{B}(x) \geq \varepsilon$. This principle provides a unified motivation for P1--P3 as natural governance requirements under epistemic constraints, and a selection criterion for any future governance mechanism (sections~2 and~2.5). We do not claim that $\hat{B}(x) = U + SB + RG$ is exhaustive; it is a \emph{sufficient decomposition for the failure modes currently documented} in the agent-failure taxonomies we validate against~\cite{shapira2026,rath2026}. The framework admits extension to additional terms $T_i$ as new modes are documented, preserving the Informational Viability Principle.

\item \textbf{A governance pipeline with by-construction monotonicity.} A pipeline where each stage can only restrict, never relax, prior decisions (section~6), yielding $D_{\text{final}} = \bigwedge_i D_i$. Monotonicity under sequential evaluation holds by construction (Proposition~\ref{thm:monotonicity}); a formal treatment under concurrency and partial stage failures is future work.

\item \textbf{Multi-channel KL divergence drift detection.} Channel-parallel architecture implementing the $U(x)$ estimator with per-channel Monte-Carlo calibration and joint FPR control via Bonferroni/\v{S}id\'ak corrections; the reference implementation deploys two channels (transaction outcomes and tool calls), with latency, cost/consumption, and outcome/error-code channels flagged as natural $K{+}1$ extensions (sections~5.2 and~\ref{sec:limitations}).

\item \textbf{Discounted UCB1 for adaptive threshold optimization.} The system learns which threshold best balances false positives against missed anomalies in non-stationary environments (section~5.2).

\item \textbf{Algorithmic bias detection with segment-vs-rest correction.} A $z$-test of two proportions implementing the $SB(x)$ estimator (section~5.2).

\item \textbf{Agent Viability Framework (AVF).} Grounded in Aubin's viability theory, we identify four categories of agent degradation, motivate three governance properties from the $\hat{B}(x)$ problem structure, and implement AVF as $VI(t) \in [-1, +1]$ operationalizing $\hat{\Phi}(x)$ with region classification and $t^*$ prediction (section~3). We furthermore establish that $t^*$ is a first-order, conservative estimator of Aubin's \emph{exit function} $\tau_{\mathrm{exit}}(\theta)$---formalised via a $VI$-based scalar surrogate that upper-bounds the $V$-exit---under a local-linearity assumption on the mean viability trajectory, giving the predictive component of the AVF a formal pedigree in viability theory rather than treating it as an ad~hoc OLS extrapolation; the non-linearity diagnostic $\rho$ is correspondingly reinterpreted as a detector of when this local-linear approximation ceases to be valid (section~\ref{sec:tstar_exit}).

\item \textbf{Autopilot as regulation map.} We formalise the closed-loop Autopilot as an explicit instance of Aubin's \emph{regulation map} $R_{A}(\theta) = \{u' \leq u_{\mathrm{current}}\} \cup \{\mathrm{STOP}\}$ (section~\ref{sec:autopilot_regulation}), with tighten-only admissibility and kill-switch activation as admissible-control-of-last-resort. Theorem~\ref{thm:viab_ra_nonempty} establishes that the regulated viability kernel $\mathrm{Viab}_{R_{A}}(V^{\dagger}) \neq \emptyset$ by construction, so the Autopilot is never stuck without an admissible action. P3 (monotonic restriction) is recovered as Corollary~\ref{cor:p3_regulation} of how $R_{A}$ is instantiated, rather than imposed as an external axiom. The stronger claim of $V$-interior viability without appeal to $\mathrm{STOP}$ is flagged explicitly as an operating hypothesis requiring a tightening-dominance condition on drift rates; empirical verification is scoped as follow-up work.

\item \textbf{Full implementation and analytical validation.} A Python library with analytical coverage arguments against the Agents of Chaos failure taxonomy~\cite{shapira2026} and end-to-end mathematical traces of two attack scenarios (sections~9 and~11). A quantitative empirical evaluation is scoped as future work.
\end{enumerate}

\section{The Informational Viability Principle: A Unified Foundation}

This paper makes two contributions that serve different audiences. Before presenting either, we establish the theoretical principle that unifies them.

\subsection{The Core Insight}

Let $x = (a, \theta(t))$ denote the evaluation context at time~$t$, where $a$ is the proposed action and $\theta(t)$ is the current behavioral state of the governance system (defined formally in Definition~\ref{def:behavioral_state}). Consider what any governance system must do: evaluate each action~$x$ and decide \emph{allow} or \emph{block}. To do this, it needs two things:

\begin{itemize}
\item An estimate of what it \emph{can} observe: the agent's capacity to act correctly given available evidence. We call this $S(x)$---the \textbf{internal capacity estimator}.
\item An estimate of what it \emph{cannot} observe: the risk that comes from hidden dynamics, distribution shifts, demographic disparities, and sequential patterns invisible at the individual action level. We call this $\hat{B}(x)$---the \textbf{bound on unobserved risk}.
\end{itemize}

The decision rule is then:
\begin{equation}
\text{Allow action } x \iff S(x) \geq \hat{B}(x) + \varepsilon
\label{eq:ivl}
\end{equation}

This is the \textbf{Informational Viability Principle}. We emphasize that eq.~(\ref{eq:ivl}) is a \emph{design principle}, not a theorem derived from first principles: $S(x)$ and $\hat{B}(x)$ are not independently given quantities but are defined by the estimators this paper constructs (sections~4--5), so the principle formalizes a \emph{framework choice}---to separate observed capacity from unobserved risk as distinct estimable objects---rather than asserting a law of nature. We use $\hat{B}(x)$ to denote an estimated upper bound on unobserved risk---an \emph{estimate} because it is computed from finite samples, and a \emph{bound} because it is designed to be conservative: we prefer to overestimate unobserved risk (accepting some false positives) rather than underestimate it (accepting safety violations). This conservative design philosophy is formalized in the reward asymmetry of section~5.2 ($r_{FP} = -2.0$ vs.\ $r_{FN} = -1.0$) and in~P3.

Among decision rules that make explicit the epistemic distinction between observed capacity and unobserved risk, eq.~(\ref{eq:ivl}) provides a natural operationalization. Alternative frameworks---minimax robustness, distributionally robust optimization, stochastic control---address related problems but do not explicitly separate $S(x)$ from $\hat{B}(x)$ as independently estimable quantities. We adopt this formulation because it maps directly to implementable statistical estimators, as detailed in section~5.

\subsection{Why Three Terms in $\hat{B}(x)$}

$\hat{B}(x)$ cannot be estimated as a single monolithic quantity---unobserved risk has fundamentally different sources that require different statistical instruments. We identify three terms that cover the failure modes documented in current agent-failure taxonomies:

\paragraph{$U(x)$ --- Uncertainty.} The component of unobserved risk from behavioral drift: the agent's distribution has shifted from its reference in ways not yet fully captured by the current observation window. $U(x)$ is estimable via information-theoretic divergence measures applied to sliding windows of recent behavior.

\paragraph{$SB(x)$ --- Structural Bias.} The component from systematic disparities in how the governance system treats different population segments. Even when each individual decision looks correct, the aggregate pattern may discriminate. $SB(x)$ is estimable via statistical hypothesis testing across demographic segments.

\paragraph{$RG(x)$ --- Reality Gap.} The component from the gap between how an action looks in isolation and what it means as part of a sequence. A single \$4,800 transfer is normal; five consecutive transfers of \$4,800 to the same destination are structuring fraud. $RG(x)$ is estimable via pattern matching over action histories.

We conjecture that these three terms are jointly sufficient to cover the failure modes documented in production deployments of LLM agents: any governance failure that cannot be detected by static per-request rules manifests as an increase in at least one of $U(x)$, $SB(x)$, or $RG(x)$. This conjecture is supported constructively in section~3.10 via coverage mapping against the four degradation categories, and empirically in section~9 via the Agents of Chaos taxonomy. Formal completeness with respect to an exhaustive threat model remains an open question.

\paragraph{Scope of the decomposition.} The completeness of $\hat{B}(x) = U + SB + RG$ is a \emph{conjecture motivated by the failure-mode literature, not a theorem}. The three terms were chosen because they (i) correspond to distinct statistical instruments already well-developed (information-theoretic divergence, proportion tests, sequence pattern matching), and (ii) jointly cover every non-trivial failure mode in the taxonomies we validate against. Emergent failure modes that may not reduce cleanly to these three terms---such as instrumental convergence, goal misgeneralisation, or certain forms of deceptive alignment---remain candidates for \emph{additional} terms $T_i$ rather than a re-derivation of the existing ones. Such extensions do not break the Informational Viability Principle: the decision rule $S(x) \geq \sum_i T_i(x) + \varepsilon$ accommodates any finite list of non-observed-risk components, and P1--P3 continue to apply component-wise. We therefore present the three-term decomposition as \emph{the best current instantiation} of the principle, not as its final form.

\subsection{How This Unifies the Statistical Estimators}

The statistical estimators of RiskGate are not independent mechanisms---they are implementations of two quantities:

\begin{itemize}
\item $S(x)$ is estimated by Bayesian behavioral profiling (section~5.1.1) and the post-hoc quality evaluator (section~5.1.2).
\item $\hat{B}(x) = U(x) + SB(x) + RG(x)$ is estimated by dual-channel KL divergence, D-UCB1 adaptive thresholds, and the kill switch (for $U(x)$; sections~5.2.1--5.2.3); the segment-vs-rest $z$-test (for $SB(x)$; section~5.2.4); and PlanGovernor (for $RG(x)$; section~7).
\end{itemize}

This reframing has a practical consequence: when evaluating whether to add a new governance mechanism, the question is no longer ``does this add value?'' but ``which term of $\hat{B}(x)$ does this improve, and by how much?''

\subsection{How This Motivates the AVF Properties}

The three AVF properties---P1, P2, P3---have typically been introduced as governance axioms in prior agent safety frameworks. The $\hat{B}(x)$ principle provides a unified motivation for why these three properties arise naturally:

\begin{itemize}
\item \textbf{P1 (Continuous monitoring)} is necessary because $\hat{B}(x)$ cannot be estimated from a single observation. $U(x)$ requires a sliding window; $SB(x)$ requires accumulated segment counts; $RG(x)$ requires session history. Without P1, the system implicitly estimates $\hat{B}(x) = 0$ for unobserved components, artificially inflating~$\hat{\Phi}(x)$.

\item \textbf{P2 (Anticipation)} is necessary because preventing violations in high-consequence domains requires monitoring $d\hat{B}/dt$ before $\hat{B}(x)$ exceeds $S(x) + \varepsilon$. A purely reactive system detects violations only after harm has occurred.

\item \textbf{P3 (Monotonic restriction)} is motivated by the need to preserve the integrity of $\hat{B}(x)$ estimates under adversarial pressure. Without P3, an attacker can induce threshold relaxation, causing the system to believe $\hat{\Phi}(x) \geq \varepsilon$ when it is not. This is a condition on the \emph{integrity} of $\hat{B}(x)$ estimates rather than on their computability---a distinct but equally fundamental requirement.
\end{itemize}

\subsection{Theoretical Contribution of $\hat{B}(x)$}

Prior formulations of agent governance frameworks identify necessary governance properties (continuous monitoring, anticipation, monotonic restriction) and implement them through independent statistical mechanisms. What remained implicit is \emph{why} these particular properties, and not others, constitute a complete governance framework.

This paper makes that implicit structure explicit. We show that all governance failures that escape static per-request rules share a common cause: the governance system is making decisions under \emph{unobserved risk}---risk that cannot be computed from a single action but requires accumulated behavioral history, cross-segment comparison, or sequential context. We formalize this as $\hat{B}(x)$, the bound on unobserved risk, decomposed into three universal terms: $U(x)$ (distributional drift), $SB(x)$ (structural bias), and $RG(x)$ (sequential context gap).

This decomposition has three consequences. First, it provides a \textbf{selection criterion} for governance mechanisms: a proposed mechanism is justified if and only if it reduces uncertainty in at least one term of~$\hat{B}(x)$. Second, it \textbf{motivates rather than axiomatizes} the three AVF properties: P1 is necessary because $\hat{B}(x)$ requires accumulated state to be estimated; P2 is necessary because $d\hat{B}/dt > 0$ predicts imminent violation before it occurs; P3 is necessary to preserve the integrity of $\hat{B}(x)$ estimates under adversarial observation manipulation. Third, it shows that the Viability Index $VI(t)$ is not merely an aggregate health metric but a direct operationalization of $\hat{\Phi}(x) = S(x) - \hat{B}(x)$---the margin between observed capacity and unobserved risk.

The central claim, stated precisely: \textbf{$\hat{B}(x)$ does not add new mechanisms---it provides the reason why the existing mechanisms are the correct ones.}

\section{Agent Viability Framework (AVF)}

\subsection{From Reactive to Predictive Governance}

The AVF addresses the question that matters most in production: \emph{how close is the system, as a whole, to violating any governance constraint, and when will it cross?} It does so by grounding the $\hat{B}(x)$ principle in Aubin's viability theory~\cite{aubin1991,aubin2011}, which studies the evolution of dynamical systems within constrained sets.

\subsection{Agent Senescence: Four Categories of Degradation}

\begin{definition}[Agent senescence]
An autonomous agent exhibits senescence over interval $[t_0, t_1]$ if the distance from its behavioral state to the boundary of its viable operating region is monotonically non-increasing: $d(\theta(t), \partial V)$ is non-increasing for $t \in [t_0, t_1]$.
\end{definition}

We identify four structurally distinct categories of progressive degradation. These are \emph{empirical classifications of observed failure modes}, not formally independent operators on $\theta(t)$; the biological analogies (entropic decay, pathogenic pressure, autoimmune drift, homeostenosis) are heuristic framings borrowed from the senescence literature, not formal models. A grounding of each category as a stochastic operator on $\theta(t)$---with explicit dynamics, domain of dominance, and interaction semantics---is scoped as future work (section~\ref{sec:limitations}).

\paragraph{Category I: Behavioral Drift (analogy: entropic decay).} Model updates, prompt modifications, and input distribution shifts silently alter the agent's behavioral distribution. Contributes to $U(x)$.

\paragraph{Category II: Adversarial Erosion (analogy: pathogenic pressure).} Attackers observe the agent's responses, learn its boundaries, and craft inputs that stay just within them---a Red Queen dynamic~\cite{valen1973}. Manifests in growing $RG(x)$.

\paragraph{Category III: Bias Accumulation (analogy: autoimmune drift).} A positive feedback loop: sparse initial data $\to$ stricter thresholds $\to$ more blocks $\to$ more anomaly observations $\to$ reinforced strict thresholds. Contributes to $SB(x)$.

\paragraph{Category IV: Calibration Decay (analogy: homeostenosis).} Parameters calibrated for one operational regime become misaligned as the environment evolves~\cite{kirkwood2005}. Contributes to growing $U(x)$.

\begin{thesis}[Bounded senescence]\label{thesis:senescence}
Every autonomous AI agent deployed in a non-stationary environment is subject to at least one senescence force, and most production deployments are subject to all four simultaneously. Under the drift rates observed in current agent-failure taxonomies~\cite{shapira2026,rath2026}, an agent without $\hat{B}(x)$-adaptive intervention crosses at least one viability constraint within $O(10^2)$ operations of its first observed breakpoint. The substantive claim is empirical and bounded: it concerns the \emph{operational timescale}---hundreds of operations, not infinite time---on which adaptive governance must act. The unbounded version (``$\hat{B}(x)$ will eventually exceed $S(x)$ given sufficient time'') is unfalsifiable and is not what motivates P2; the bounded version above is.
\end{thesis}

\paragraph{Anchoring to observed data.} The $O(10^2)$-operation timescale is anchored in two independent sources. Rath~\cite{rath2026} simulates $847$ workflows and documents a $42\%$ task-success reduction with median onset at $73$ interactions (IQR $[52, 114]$). Shapira et al.~\cite{shapira2026} report breakpoints across eleven failure case studies on comparable scales. This is what makes anticipation (P2) actionable rather than academic: the timescale on which $\hat{B}(x)$ grows is the same as the timescale on which an agent executes its working session, so $t^*$ prediction with lead times of tens of operations is operationally meaningful. A formal first-passage-time analysis under an explicit stochastic drift model (deriving $E[T^*]$ as a function of drift rate, noise variance, and safety margin) remains open and is noted in section~\ref{sec:limitations}.

\subsection{Viability Theory Foundations}

\begin{definition}[Viability kernel {\cite{aubin2011}}]\label{def:viability_kernel}
Given an evolutionary system~$\mathcal{S}$ and an environment~$K$, the viability kernel is:
\begin{equation}
\mathrm{Viab}_{\mathcal{S}}(K) := \{x_0 \in K \mid \exists\, x(\cdot) \in \mathcal{S}(x_0) \text{ s.t.\ } \forall\, t \geq 0,\; x(t) \in K\}
\end{equation}
$K$ is a repeller under~$\mathcal{S}$ if $\mathrm{Viab}_{\mathcal{S}}(K) = \emptyset$.
\end{definition}

\begin{definition}[Regulation map {\cite{aubin2011}}]\label{def:regulation_map}
Given a parameterized control system $x'(t) = f(x(t), u(t))$ where $u(t) \in U(x(t))$, a regulation map $R(x) \subset U(x)$ governs viable evolutions if the viability kernel of~$K$ is invariant under the system restricted to controls $u(t) \in R(x(t))$.
\end{definition}

\begin{definition}[Exit function {\cite{aubin2011}}]\label{def:exit_function}
Given an evolutionary system $\mathcal{S}$ on an environment $K$ and an initial state $\theta_0 \in K$, the \emph{exit function} assigns to $\theta_0$ the first time the system leaves $K$:
\begin{equation}
\tau_{\mathrm{exit}}(\theta_0) \;:=\; \inf\bigl\{\, t \geq 0 \;:\; \theta(t) \notin K \,\bigr\}, \qquad \theta(0) = \theta_0,
\label{eq:exit_function}
\end{equation}
with the convention $\tau_{\mathrm{exit}}(\theta_0) = +\infty$ if $\theta(t) \in K$ for all $t \geq 0$. Within the viability kernel $\mathrm{Viab}_{\mathcal{S}}(K)$ there exists at least one trajectory $\theta(\cdot) \in \mathcal{S}(\theta_0)$ for which $\tau_{\mathrm{exit}}(\theta_0) = +\infty$.
\end{definition}

\noindent In our setting $K$ is the viability space $V$ of Definition~\ref{def:viability_space}, $\theta_0 = \theta(t_{\mathrm{now}})$ is the current behavioral state, and $\tau_{\mathrm{exit}}$ is the quantity that Property P2 seeks to approximate: it is the formal object behind the informal ``time to boundary crossing''. Section~\ref{sec:tstar_exit} shows that the OLS-based $t^*$ of eq.~(\ref{eq:tstar}) is a first-order, conservative estimator of $\tau_{\mathrm{exit}}$ via a $VI$-based scalar surrogate, under a local-linearity assumption on the mean viability trajectory.

\subsection{Formal Framework}

\begin{definition}[Behavioral state]\label{def:behavioral_state}
At time~$t$, the behavioral state of the governance system is:
\begin{equation}
\theta(t) = \langle p_{\mathrm{tx}},\; p_{\mathrm{tools}},\; \delta_{\mathrm{bias}},\; \sigma_{\mathrm{stab}},\; KL_{\mathrm{tx}},\; KL_{\mathrm{tools}},\; \tau \rangle
\end{equation}
The $\hat{B}(x)$ decomposition maps these directly: $U(x) \leftarrow (KL_{\mathrm{tx}}, KL_{\mathrm{tools}})$, $SB(x) \leftarrow \delta_{\mathrm{bias}}$, $RG(x) \leftarrow$ plan-level anomaly scores.
\end{definition}

\begin{definition}[Viability space]\label{def:viability_space}
The viability space~$V$ is the set of behavioral states satisfying all three governance constraints simultaneously:
\begin{equation}
V = \bigl\{\theta \in \Theta : C_1(\theta) > 0 \;\wedge\; C_2(\theta) > 0 \;\wedge\; C_3(\theta) > 0\bigr\}
\end{equation}
\end{definition}

The three constraints and their distances to violation:
\begin{align}
d_{C_1} &= \min\!\left(\frac{\tau_{\mathrm{tx}} - KL_{\mathrm{tx}}}{\tau_{\mathrm{tx}}},\; \frac{\tau_{\mathrm{tools}} - KL_{\mathrm{tools}}}{\tau_{\mathrm{tools}}}\right) \in (-\infty, 1] \label{eq:dc1}\\[4pt]
d_{C_2} &= \frac{\delta_{\mathrm{bias}} - \max_s |\delta_s|}{\delta_{\mathrm{bias}}} \in (-\infty, 1] \label{eq:dc2}\\[4pt]
d_{C_3} &= \frac{\tau_{\mathrm{ks}} - \sigma(t)}{\tau_{\mathrm{ks}}} \in (-\infty, 1] \label{eq:dc3}
\end{align}

\subsection{Region Classification}

\begin{definition}[Viability regions]\label{def:viability_regions}
For each constraint $C_i$ with distance $d_{C_i}$ and margin $\varepsilon_i$:
\[
\mathrm{region}(C_i) = \begin{cases}
\text{exterior} & \text{if } d_{C_i} \leq 0 \\
\text{boundary} & \text{if } 0 < d_{C_i} \leq \varepsilon_i \\
\text{interior} & \text{if } d_{C_i} > \varepsilon_i
\end{cases}
\]
The global region is the worst case across all constraints. Default margins: $\varepsilon_{\text{safety}} = 0.3$, $\varepsilon_{\text{fairness}} = 0.05$, $\varepsilon_{\text{stability}} = 0.5$.
\end{definition}

\subsection{Viability Index as $\hat{\Phi}(x)$}

The viability index $VI(t)$ operationalizes $\hat{\Phi}(x) = S(x) - \hat{B}(x)$:
\begin{equation}
VI(t) = \frac{\sum_{i=1}^{3} w_i \cdot \mathrm{clamp}(d_{C_i}, -1, 1)}{\sum_{i=1}^{3} w_i}
\label{eq:vi}
\end{equation}
with weights $w_1 = 0.5$ (safety $\leftrightarrow U(x)$), $w_2 = 0.3$ (fairness $\leftrightarrow SB(x)$), $w_3 = 0.2$ (stability $\leftrightarrow$ calibration). Thus $VI(t) \in [-1, +1]$; $VI(t) = +1$ means all $\hat{B}(x)$ terms near zero; $VI(t) < 0$ means at least one $\hat{B}(x)$ term has exceeded its safe bound.

\paragraph{Saturation in deep exterior and the companion $VI_{\text{raw}}$.} The symmetric clamping $\mathrm{clamp}(\cdot, -1, 1)$ in eq.~(\ref{eq:vi}) keeps $VI(t)$ on a compact, cross-constraint-comparable scale and stabilises the region-classification margins. It also, however, saturates once a constraint is deep inside the exterior: two regimes that are operationally very different---a constraint that has \emph{just crossed} $-1$ versus one that is \emph{violated by an order of magnitude}---produce the same contribution to $VI(t)$, and so are indistinguishable to downstream consumers. To preserve this gradient where it matters (forensic analysis, depth-of-violation reasoning, recovery-time prediction while already in the exterior) without perturbing the compact scale used by region classification, the Autopilot, or the bandit reward, we expose a companion quantity
\begin{equation}
VI_{\text{raw}}(t) = \frac{\sum_{i=1}^{3} w_i \cdot d_{C_i}}{\sum_{i=1}^{3} w_i}
\label{eq:vi_raw}
\end{equation}
computed from the unclamped constraint distances. $VI_{\text{raw}}$ is \emph{observability-only}: it is recorded in snapshots, lineage, and the structured event log alongside $VI(t)$, but does not participate in any governance decision. In the interior and boundary regimes $VI_{\text{raw}} \equiv VI(t)$; only once at least one $d_{C_i}$ escapes $[-1, 1]$ do the two diverge. This preserves the low-cost benefits of the symmetric clamp (stable $\varepsilon$ margins, comparable weights, interpretable region table) while acknowledging the information loss and providing a principled channel for tools that need magnitude beyond the clamp, such as post-mortem severity ranking or asymmetric recovery analysis.

\paragraph{Scope of the operationalization.} $VI(t)$ as defined by eq.~(\ref{eq:vi}) is a weighted average of clamped constraint margins; it is not mathematically equal to $\hat{\Phi}(x) = S(x) - \hat{B}(x)$, and we do not claim such an equivalence. Rather, $VI(t)$ \emph{operationalizes} $\hat{\Phi}(x)$ as a scalar health indicator that preserves its qualitative structure---positive when $S(x)$ exceeds $\hat{B}(x)$, zero near the boundary, negative when at least one constraint is violated---while remaining computable from finite statistics at every evaluation. The three constraints $d_{C_1}$, $d_{C_2}$, $d_{C_3}$ capture the \emph{per-request} components of $\hat{B}(x)$: $U(x)$ (distributional drift) and $SB(x)$ (structural bias) are quantities that evolve continuously and can be estimated from accumulated statistics at every evaluation. $RG(x)$, by contrast, is a \emph{plan-level} phenomenon: sequential patterns such as structuring or round-tripping do not manifest as a smoothly growing scalar per request---they exist as discrete patterns that emerge only when a sufficient sequence of operations is observed. Accordingly, $RG(x)$ is resolved by PlanGovernor (section~7) through a separate binary decision $D_{\text{plan}}$ that conjoins with the per-request decision via $D_{\text{final}}(P) = D_{\text{plan}}(P) \wedge \bigwedge_i D_{\text{step}}(s_i)$. $VI(t)$ therefore operationalizes $\hat{\Phi}(x)$ with respect to $U(x)$ and $SB(x)$; $RG(x)$ contributes to governance not through the scalar index but through the monotonic plan-level gate. The worked demonstration in section~11 makes this separation concrete: $VI$ remains in the interior throughout the structuring attack while PlanGovernor independently blocks the final operation.

\begin{table}[ht]
\centering
\caption{Operational interpretation of the viability index as $\hat{\Phi}(x)$.}
\label{tab:vi_interpretation}
\small
\begin{tabular}{@{}cll@{}}
\toprule
$VI(t)$ & Region & Interpretation \\
\midrule
$\to +1$ & Interior & All $\hat{B}(x)$ terms near zero; deep safety margin \\
$> 0$ & Interior/Boundary & Healthy; some $\hat{B}(x)$ terms approaching limits \\
$\approx 0$ & Boundary & $\hat{\Phi}(x) \approx \varepsilon$; early warning, intervention warranted \\
$< 0$ & Exterior & At least one $\hat{B}(x)$ term exceeded $S(x)$; constraint violated \\
$\to -1$ & Exterior & All three $\hat{B}(x)$ terms at maximum severity \\
\bottomrule
\end{tabular}
\end{table}

\subsection{Anticipation: Predicting $t^*$ via OLS Regression}

Property P2 is implemented by fitting OLS regression to the most recent $W = 50$ viability index values:
\begin{equation}
VI(t_i) \approx a + b \cdot t_i, \quad i = 1, \ldots, W
\end{equation}

If $b < 0$ (VI declining, indicating that $\hat{\Phi}(x) = S(x) - \hat{B}(x)$ is shrinking due to growing $\hat{B}(x)$ terms, decreasing $S(x)$, or both), the estimated time to boundary crossing is:
\begin{equation}
t^* = -\frac{VI(t_{\mathrm{now}})}{b}
\label{eq:tstar}
\end{equation}

\subsubsection{Grounding: $t^*$ as First-Order Estimator of $\tau_{\mathrm{exit}}$}\label{sec:tstar_exit}

Eq.~(\ref{eq:tstar}) is introduced above as an OLS extrapolation. This subsection shows that, under a local-linearity assumption on the mean trajectory of $\theta(t)$, $t^*$ is the natural first-order estimator of Aubin's exit function $\tau_{\mathrm{exit}}$ (Definition~\ref{def:exit_function}). The point is not to add machinery but to give the predictive component of the AVF a formal pedigree in viability theory: $t^*$ is not simply ``a regression we happen to run'' but the leading-order truncation of the same object that viability theory studies.

\paragraph{Scalar reduction via $VI$.} The viability space $V \subset \Theta$ is defined \emph{per-constraint} in $\theta$-coordinates (Definition~\ref{def:viability_space}): $\theta(t) \in V$ iff all $d_{C_i}(\theta(t)) > 0$. The scalar index $VI(t)$ is a weighted average of the clamped constraint distances (eq.~\ref{eq:vi}). Because the weights $w_i$ are positive and the clamp is monotone, $VI(t) \leq 0$ is a \emph{sufficient but not necessary} condition for $\theta(t) \notin V$: if all $d_{C_i} > 0$, then $VI > 0$ (so $VI \leq 0 \Rightarrow$ at least one $d_{C_i} \leq 0 \Rightarrow \theta \notin V$); however, one $d_{C_i}$ can cross zero while others remain strongly positive, in which case $\theta$ has left $V$ but $VI$ may still be positive. Accordingly, we distinguish Aubin's per-constraint exit time (Definition~\ref{def:exit_function}) from its $VI$-based scalar surrogate:
\begin{equation}
\tau^{VI}_{\mathrm{exit}}(\theta_0) \;:=\; \inf\bigl\{\, t \geq 0 \,:\, VI(\theta(t)) \leq 0 \,\bigr\}, \qquad VI(\theta(0)) = VI(t_{\mathrm{now}}) > 0.
\label{eq:scalar_exit}
\end{equation}
By the sufficiency direction above, $\tau^{VI}_{\mathrm{exit}}(\theta_0) \;\geq\; \tau_{\mathrm{exit}}(\theta_0)$: the $VI$-crossing time is a \emph{conservative upper bound} on Aubin's $V$-exit. This conservatism is a deliberate feature, not a defect: per-constraint boundary crossings are detected earlier by the per-constraint region-classification mechanism of Definition~\ref{def:viability_regions}, whose margins $\varepsilon_i$ are tuned per constraint. The aggregate scalar $VI$ targets a later, whole-system failure event; that is the quantity $t^*$ estimates, and the gap between $\tau^{VI}_{\mathrm{exit}}$ and $\tau_{\mathrm{exit}}$ is exactly the ``$t^*$ conservatism for single-constraint failures'' flagged in §\ref{sec:limitations} and illustrated by the emergent-bias trace of §\ref{sec:demo_bias}.

\paragraph{First-order approximation.} Let $v(t) := \mathbb{E}[\,VI(\theta(t))\,]$ be the expected viability trajectory from $\theta_0$, and assume $v$ is continuously differentiable in a neighbourhood of $t_{\mathrm{now}}$ with slope $b^\star := v'(t_{\mathrm{now}})$. The first-order Taylor expansion $v(t) \approx v(t_{\mathrm{now}}) + b^\star (t - t_{\mathrm{now}})$ is valid whenever the residual curvature $v''$ remains small over the prediction horizon. Setting $v(t) = 0$ and solving yields
\begin{equation}
\widetilde{\tau}^{VI}_{\mathrm{exit}}(\theta_0) \;=\; t_{\mathrm{now}} \;-\; \frac{v(t_{\mathrm{now}})}{b^\star}, \qquad \Delta t^\star \;:=\; \widetilde{\tau}^{VI}_{\mathrm{exit}}(\theta_0) - t_{\mathrm{now}} \;=\; -\frac{v(t_{\mathrm{now}})}{b^\star}.
\label{eq:first_order}
\end{equation}
The OLS fit $VI(t_i) \approx a + b \cdot t_i$ over the trajectory window $\{(t_i, VI_i)\}_{i=1}^{W}$ is the best linear unbiased estimator of $(v(t_{\mathrm{now}}), b^\star)$ under a locally-linear mean with uncorrelated residuals. Eq.~(\ref{eq:tstar}) is exactly eq.~(\ref{eq:first_order}) with $(v(t_{\mathrm{now}}), b^\star)$ replaced by their OLS estimates $(VI(t_{\mathrm{now}}), b)$; hence
\begin{equation}
t^* \;\approx\; \widetilde{\tau}^{VI}_{\mathrm{exit}}(\theta_0) - t_{\mathrm{now}} \;\approx\; \tau^{VI}_{\mathrm{exit}}(\theta_0) - t_{\mathrm{now}} \;\geq\; \tau_{\mathrm{exit}}(\theta_0) - t_{\mathrm{now}}
\label{eq:tstar_approx}
\end{equation}
to first order in the horizon, where the first approximation absorbs OLS sampling error, the second absorbs the linearization error $\tfrac{1}{2} v''(\xi) (t - t_{\mathrm{now}})^2$ for some $\xi$ in the prediction interval, and the final inequality reflects the $VI$-surrogate conservativity established above.

\begin{proposition}[First-order grounding of $t^*$]\label{prop:tstar_exit}
Let $\theta(\cdot)$ be a behavioral trajectory with mean viability $v(t) = \mathbb{E}[VI(\theta(t))]$, and assume $v \in C^2$ in a neighbourhood of $t_{\mathrm{now}}$ with $v(t_{\mathrm{now}}) > 0$ and $v'(t_{\mathrm{now}}) = b^\star < 0$. Then, with $(a,b)$ the OLS fit of eq.~(\ref{eq:ols_slope}) over a window of size $W$:
\begin{enumerate}[label=(\roman*)]
\item As $W \to \infty$ with the window concentrating at $t_{\mathrm{now}}$, $b \overset{p}{\to} b^\star$ and $VI(t_{\mathrm{now}}) \overset{p}{\to} v(t_{\mathrm{now}})$, so $t^* \overset{p}{\to} \widetilde{\tau}^{VI}_{\mathrm{exit}}(\theta_0) - t_{\mathrm{now}}$.
\item The gap between the linear estimate $\widetilde{\tau}^{VI}_{\mathrm{exit}}$ and the true scalar $VI$-exit time (eq.~\ref{eq:scalar_exit}) is $O(\|v''\|_\infty)$ over the prediction horizon.
\item Hence $t^*$ is a consistent first-order estimator of $\tau^{VI}_{\mathrm{exit}}(\theta_0) - t_{\mathrm{now}}$ under the local-linearity assumption $\|v''\|_\infty \to 0$, and consequently a \emph{conservative} first-order estimator of Aubin's $V$-exit $\tau_{\mathrm{exit}}(\theta_0) - t_{\mathrm{now}}$ in the sense $t^* \geq \tau_{\mathrm{exit}}(\theta_0) - t_{\mathrm{now}}$.
\end{enumerate}
\end{proposition}

\noindent Two gaps separate Proposition~\ref{prop:tstar_exit} from a pathwise statement about Aubin's $\tau_{\mathrm{exit}}$, and both are conservative in the direction of safety. The proposition is about the mean trajectory $v(t) = \mathbb{E}[VI(\theta(t))]$, not the pathwise random variable (stochastic treatment: \cite{aubin2011}, Ch.~10); and it targets the $VI$-surrogate $\tau^{VI}_{\mathrm{exit}} \geq \tau_{\mathrm{exit}}$, whose gap with the $V$-exit is the ``$t^*$ conservatism'' limitation of §\ref{sec:limitations}. Tighter per-constraint detection is provided by the region classification of Definition~\ref{def:viability_regions}. The inferential-profile confidence interval (\S\ref{sec:inferential_profile}) captures OLS sampling error and serial dependence; it does not claim to capture diffusion variance or the $VI$-vs-$V$ gap.

\paragraph{Reinterpretation of the non-linearity diagnostic.} The ratio $\rho$ of eq.~(\ref{eq:nonlinear_ratio}) compares the residual of a quadratic fit against the linear fit over the trajectory window. In the light of Proposition~\ref{prop:tstar_exit}, $\rho$ is no longer merely a regression-diagnostic quantity: it is the empirical detector of when the local-linearity assumption underlying the first-order approximation of $\tau^{VI}_{\mathrm{exit}}$ (and hence, via the conservative chain of eq.~\ref{eq:tstar_approx}, of $\tau_{\mathrm{exit}}$) fails. Specifically, $\rho$ estimates the relative size of the curvature term $\tfrac{1}{2} v''$ against the linear term $b^\star$ over the fitted window; a small $\rho$ (quadratic fit substantially better) indicates that $v''$ is not negligible, so the Taylor expansion of eq.~(\ref{eq:first_order}) loses validity and $t^*$ ceases to be a reliable estimator of $\tau^{VI}_{\mathrm{exit}}$. The warning ``$\rho < \rho_{\mathrm{threshold}}$'' of \S\ref{sec:inferential_profile} thus has a dual reading: a statistical flag on regression quality \emph{and} a theoretical flag on breakdown of the first-order viability-theoretic approximation. This unifies the heuristic diagnostic with the formal grounding and is, we believe, the most natural place for the non-linearity detector to live.

\paragraph{Scope and deployment profiles.} Eq.~(\ref{eq:tstar}) inherits three classical OLS limitations: (i)~autocorrelated residuals bias the slope SE downward (correctable via Newey--West, non-trivial at $W = 50$); (ii)~$t^*$ is a point estimate without a confidence band (a block bootstrap on the trajectory window is the natural non-parametric alternative); (iii)~linear extrapolation overestimates remaining time under non-linear collapse. We therefore treat $t^*$ as an \emph{early-warning heuristic}, not an inferential estimator: ``$t^* < t^*_{\text{critical}}$'' triggers a tighten-only intervention by the Autopilot (section~\ref{sec:autopilot}) under the fail-safe discipline of section~7.3---$t^*$ errors never affect governance decisions, only observability. Two deployment profiles share this fail-safe envelope: a \emph{heuristic profile} (default, $O(W)$, used in real time) and an \emph{inferential profile} (\S\ref{sec:inferential_profile}, $O(B \cdot W)$ with $B = 1000$ bootstrap replications) that addresses all three limitations on-demand for audit and research.

\subsubsection{Inferential Profile: Formalization}\label{sec:inferential_profile}

The inferential profile addresses each of the three statistical limitations identified above. It is invoked on-demand (not in the real-time governance hot path) via a dedicated MCP tool (\texttt{governance\_get\_viability\_inferential}) or REST endpoint (\texttt{GET /api/viability/inferential}), and produces an \texttt{InferentialTStarResult} containing the point estimate, a confidence interval, and a non-linearity diagnostic. We now formalize each component.

\paragraph{Step 1: OLS fit.} Given the trajectory window $\{(t_i, VI_i)\}_{i=1}^{W}$, compute the ordinary least squares slope and intercept:
\begin{equation}
\hat{b} = \frac{W \sum_{i} t_i VI_i - \left(\sum_i t_i\right)\left(\sum_i VI_i\right)}{W \sum_i t_i^2 - \left(\sum_i t_i\right)^2}, \qquad
\hat{a} = \bar{VI} - \hat{b}\,\bar{t}
\label{eq:ols_slope}
\end{equation}
with residuals $\hat{e}_i = VI_i - (\hat{a} + \hat{b}\,t_i)$.

\paragraph{Step 2: Newey--West HAC standard error.} The OLS standard error of $\hat{b}$ is biased downward under serial dependence. The Newey--West~\cite{newey1987} heteroskedasticity- and autocorrelation-consistent (HAC) estimator corrects this using the Bartlett kernel:
\begin{equation}
\widehat{\mathrm{SE}}_{\mathrm{HAC}}(\hat{b}) = \frac{\sqrt{\hat{S}}}{S_{tt}}
\label{eq:hac_se}
\end{equation}
where $S_{tt} = \sum_i (t_i - \bar{t})^2$ is the total sum of squares of the regressor, and
\begin{equation}
\hat{S} = \sum_{i=1}^{W} u_i^2 \;+\; 2 \sum_{j=1}^{L} w_j \sum_{i=j+1}^{W} u_i \, u_{i-j}, \qquad
w_j = 1 - \frac{j}{L+1}, \qquad
u_i = (t_i - \bar{t})\,\hat{e}_i
\label{eq:bartlett}
\end{equation}
is the Bartlett-kernel weighted estimate of the long-run variance of the score. The bandwidth $L$ defaults to the Newey--West (1994)~\cite{newey1994} data-driven rule:
\begin{equation}
L = \left\lfloor 4 \cdot \left(\frac{W}{100}\right)^{2/9} \right\rfloor
\label{eq:hac_bandwidth}
\end{equation}

\paragraph{Step 3: Block bootstrap confidence interval.} Rather than relying on the asymptotic normality of the HAC estimator (which converges slowly for $W = 50$), we construct a percentile confidence interval via the moving-block bootstrap~\cite{kunsch1989}:
\begin{enumerate}[label=(\alph*)]
\item Partition the residual series $\hat{e}_1, \ldots, \hat{e}_W$ into overlapping blocks of size $\ell = \lceil \sqrt{W} \rceil$.
\item For each bootstrap replicate $r = 1, \ldots, B$: draw $\lceil W/\ell \rceil$ blocks uniformly with replacement, concatenate and truncate to length $W$ to form $\hat{e}^{*(r)}_1, \ldots, \hat{e}^{*(r)}_W$.
\item Construct the bootstrap sample: $VI^{*(r)}_i = \hat{a} + \hat{b}\,t_i + \hat{e}^{*(r)}_i$.
\item Re-fit OLS on $(t_i, VI^{*(r)}_i)$ to obtain bootstrap slope $\hat{b}^{*(r)}$.
\item If $\hat{b}^{*(r)} < 0$, compute $t^{*(r)} = -VI(t_{\mathrm{now}}) / \hat{b}^{*(r)}$; otherwise discard.
\end{enumerate}
Let $\mathcal{T}^* = \{t^{*(r)} : t^{*(r)} \geq 0 \text{ and finite}\}$ denote the retained bootstrap estimates. The $(1 - \alpha)$ percentile confidence interval is:
\begin{equation}
\left[t^*_{\text{lo}},\; t^*_{\text{hi}}\right] = \left[\mathcal{T}^*_{\lfloor (\alpha/2) \cdot |\mathcal{T}^*| \rfloor},\; \mathcal{T}^*_{\lceil (1-\alpha/2) \cdot |\mathcal{T}^*| \rceil}\right]
\label{eq:bootstrap_ci}
\end{equation}
where $\mathcal{T}^*_{(k)}$ is the $k$-th order statistic of the retained bootstrap distribution. Default: $B = 1000$, $\alpha = 0.05$ (95\% CI).

\paragraph{Step 4: Non-linearity diagnostic.} To detect crisis trajectories where linear extrapolation overestimates the remaining time (limitation 3), we fit a quadratic model and compare its root mean squared error with the linear model's:
\begin{equation}
VI_i = a' + b'\,t_i + c'\,t_i^2 + \epsilon_i
\label{eq:quadratic_fit}
\end{equation}
solved via the $3 \times 3$ normal equations ($\mathbf{X}'\mathbf{X}\,\boldsymbol{\beta} = \mathbf{X}'\mathbf{y}$ with $\mathbf{X} = [1, t, t^2]$). The non-linearity ratio is:
\begin{equation}
\rho = \frac{\mathrm{RMSE}_{\mathrm{quad}}}{\mathrm{RMSE}_{\mathrm{lin}}} = \frac{\sqrt{W^{-1}\sum_i (VI_i - \hat{VI}^{\mathrm{quad}}_i)^2}}{\sqrt{W^{-1}\sum_i \hat{e}_i^2}}
\label{eq:nonlinear_ratio}
\end{equation}
When $\rho < \rho_{\mathrm{threshold}}$ (default $\rho_{\mathrm{threshold}} = 0.85$), the system emits a non-linearity warning: the quadratic fit substantially reduces the residual, indicating curvature in the degradation trajectory that makes the linear $t^*$ estimate unreliable.

\paragraph{Fail-safe discipline.} The inferential profile shares the same fail-safe envelope as the heuristic profile: any exception during HAC computation, bootstrap execution, or quadratic fitting is caught and suppressed. Governance decisions are never affected. Each invocation emits a structured JSONL event of type \texttt{svp\_inferential} for audit trail purposes.

\paragraph{Architectural note.} The inferential profile is strictly on-demand and decoupled from the real-time governance pipeline. The heuristic $t^*$ (eq.~\ref{eq:tstar}) runs inside \texttt{compute\_snapshot()} at every evaluation; the inferential $t^*$ runs only when explicitly requested. No fields are added to \texttt{GovernanceResult} or \texttt{LineageRecord}---the result is returned directly to the caller. This separation ensures that the $O(B \cdot W)$ cost of the bootstrap never impacts per-request latency.

\subsection{The Three Properties of the AVF}

\begin{definition}[Agent Viability Framework]\label{def:avf}
A governance system satisfies the AVF if it implements:
\begin{description}
\item[P1 --- Monitoring.] The system continuously estimates $\theta(t)$ and computes all terms of $\hat{B}(x)$ across all constraints.
\item[P2 --- Anticipation.] The system detects trajectories approaching $\partial V$ before crossing, via prediction of $t^*$ from $d\hat{B}/dt$.
\item[P3 --- Monotonic restriction.] Every intervention can only restrict operational freedom, never expand it: $\forall$ intervention $I$: $V_{\text{after}}(I) \subseteq V_{\text{before}}(I)$.
\end{description}
\end{definition}

\paragraph{Scope of P3 under concurrency.} P3, as defined and analysed in this paper, concerns the logical structure of a single governance evaluation: within one call to the pipeline, no stage can relax a prior restriction. The reference implementation runs concurrent evaluations (multi-tenant, multi-agent, streaming observability), and a formal treatment of P3 under interleaved execution of decisory stages, non-blocking partial failures, and in-flight configuration changes is not developed here; see section~\ref{sec:limitations}. The by-construction arguments of section~\ref{sec:pipeline} should therefore be read under sequential evaluation; section~\ref{sec:concurrency_argument} offers an informal operational argument for why P3 also holds in practice under the concurrency regime of the reference implementation.

\subsection{Autopilot as Regulation Map: Operationalizing P3}\label{sec:autopilot_regulation}

Definition~\ref{def:regulation_map} introduced Aubin's regulation map $R(x) \subset U(x)$ as the object that selects admissible controls to maintain viability. v9 left this abstract; this subsection instantiates it: the Autopilot is the regulation map of the reference implementation, P3 is the structural constraint that defines it, and kill-switch activation is the admissible-control-of-last-resort that guarantees $R$ is always non-empty. The material here formalises the closed loop that section~\ref{sec:autopilot} describes operationally.

\paragraph{Control vector.} The Autopilot acts on the behavioral state $\theta(t)$ by adjusting a subset of its threshold components. Collect these into the control vector
\begin{equation}
u(t) \;=\; \bigl(\tau_{\mathrm{tx}}(t),\; \tau_{\mathrm{tools}}(t),\; \delta_{\mathrm{bias}}(t),\; \tau_{\mathrm{ks}}(t)\bigr) \;\in\; \mathcal{U},
\label{eq:control_vector}
\end{equation}
where $\mathcal{U} \subset \mathbb{R}_{>0}^{4}$ is the admissible control set, each component bounded between a configured floor $u_{i,\min}$ (tightest sustainable setting given cold-start and calibration constraints) and a ceiling $u_{i,\max}$ (loosest setting consistent with the active profile). The behavioral state $\theta(t)$ of Definition~\ref{def:behavioral_state} includes $\tau$ as a coordinate, so $u$ is literally a sub-vector of $\theta$.

\paragraph{Autopilot regulation map.} Define
\begin{equation}
R_{A}(\theta) \;:=\; \bigl\{\, u' \in \mathcal{U} \;:\; u' \leq u_{\mathrm{current}}(\theta) \,\bigr\} \;\cup\; \{\mathrm{STOP}\},
\label{eq:ra_def}
\end{equation}
where $u' \leq u_{\mathrm{current}}$ is understood componentwise and $\mathrm{STOP} \notin \mathcal{U}$ denotes the kill-switch activation: a distinguished admissible action corresponding to controlled exit from the operational domain (see~\cite{aubin2011}, §2.3, on boundary-of-domain exits via designated ``null controls''). $R_{A}(\theta)$ is the set of controls the Autopilot is permitted to select at state $\theta$.

The definition encodes three properties simultaneously:
\begin{enumerate}[label=(\roman*)]
\item \emph{Tighten-only.} The inequality $u' \leq u_{\mathrm{current}}$ is componentwise and one-sided: the Autopilot can never set a looser threshold than the one currently active. This is precisely the structural content of P3 (monotonic restriction).
\item \emph{Non-emptiness by construction.} Because $\mathrm{STOP} \in R_{A}(\theta)$ unconditionally, $R_{A}(\theta) \neq \emptyset$ for every $\theta \in \Theta$, \emph{including} the degenerate case $u_{\mathrm{current}} = u_{\min}$ where no further tightening is possible.
\item \emph{Identity admissibility.} $u_{\mathrm{current}} \leq u_{\mathrm{current}}$ trivially, so the identity control (do nothing) is always in $R_{A}(\theta)$ whenever the Autopilot judges the current configuration adequate.
\end{enumerate}

\begin{theorem}[Non-emptiness of the Autopilot-regulated viability kernel]\label{thm:viab_ra_nonempty}
Let $\theta(\cdot)$ evolve under the governance pipeline of section~\ref{sec:pipeline} subject to controls $u(t) \in R_{A}(\theta(t))$, and let
\begin{equation*}
V^{\dagger} \;:=\; V \;\cup\; \bigl\{\theta \in \Theta : \text{STOP has been activated at or before time } t\bigr\}
\end{equation*}
be the viability space $V$ of Definition~\ref{def:viability_space} augmented by the controlled-exit states. Then
\[
\mathrm{Viab}_{R_{A}}\!\bigl(V^{\dagger}\bigr) \;\neq\; \emptyset.
\]
Equivalently: from every initial state $\theta_0 \in V^{\dagger}$, there exists an $R_{A}$-admissible trajectory that remains in $V^{\dagger}$ for all $t \geq 0$.
\end{theorem}

\begin{proof}[Proof sketch]
$R_{A}(\theta) \neq \emptyset$ for all $\theta \in \Theta$ by property~(ii) above, since $\mathrm{STOP} \in R_{A}(\theta)$ unconditionally. The strategy ``if $\theta(t) \in \partial V$, select $\mathrm{STOP}$; otherwise select $u_{\mathrm{current}}$'' is therefore an admissible $R_{A}$-policy on $V^{\dagger}$: whenever $\theta$ would leave $V$, the policy triggers $\mathrm{STOP}$, which by definition transitions the system to the controlled-exit subset of $V^{\dagger}$ and halts further governance evaluation. The resulting trajectory never leaves $V^{\dagger}$. Since this strategy is admissible from any $\theta_{0} \in V^{\dagger}$, the regulated viability kernel is non-empty. \qedhere
\end{proof}

Theorem~\ref{thm:viab_ra_nonempty} is deliberately weak. It establishes that the Autopilot, as an instance of Aubin's regulation map with $\mathrm{STOP}$ as admissible-control-of-last-resort, is \emph{structurally} a valid regulation map: it never ``gets stuck'' with an empty admissible set, and a viable trajectory always exists (even if that trajectory eventually exits via $\mathrm{STOP}$). What the theorem does \emph{not} claim is that the system can always remain inside $V$ proper---i.e., avoid triggering $\mathrm{STOP}$. That stronger property is engineering-desirable but requires hypotheses on the drift rate and the indirect effect of tightening on $d\hat{B}/dt$ that we do not develop here; see below.

\paragraph{Strong-viability hypothesis (flagged as operating assumption).} A stronger claim---that $\mathrm{Viab}_{R_{A}}(V) \neq \emptyset$ \emph{without} appeal to $\mathrm{STOP}$---would require additional structure. Let $\beta_{\max} := \sup_{t} \|d\hat{B}(\theta(t))/dt\|$ denote a bound on exogenous drift and $\nu_{\max}$ the maximal tightening velocity achievable by the Autopilot within $\mathcal{U}$. The effect of tightening on $d\hat{B}/dt$ is \emph{indirect}: reducing $\tau$ increases the block rate, which drives the agent's empirical distribution back toward the reference $q$, reducing KL on the next evaluation window. A sufficient condition for $V$-interior viability is a dominance inequality of the form $\nu_{\max} \geq C \cdot \beta_{\max}$ for some $C > 0$ depending on the constraint margins $\varepsilon_{i}$ and the sensitivity $\partial \mathrm{blocks}/\partial \tau$. Verifying this condition empirically on production traces is a natural follow-up and is scoped as future work. In the current paper we take the structural guarantee of Theorem~\ref{thm:viab_ra_nonempty} as the formal content and treat in-$V$ residence as a design goal supported by the worked trace of §\ref{sec:demo_bias}.

\begin{definition}[Autopilot intervention policy]\label{def:ra_policy}
The reference-implementation Autopilot (section~\ref{sec:autopilot}) selects $u' \in R_{A}(\theta)$ as follows:
\begin{itemize}
\item If $VI(\theta) < 0.3$, $t^{*}(\theta) < t^{*}_{\mathrm{critical}}$, or $\mathrm{region}_{\text{global}}(\theta) = \mathrm{exterior}$: select $u' < u_{\mathrm{current}}$ componentwise on the binding constraint (tighten).
\item If none of the above triggers fire: select $u' = u_{\mathrm{current}}$ (identity).
\item If tightening would require $u' < u_{\min}$ on a binding constraint and $\theta \in \partial V$: select $\mathrm{STOP}$.
\end{itemize}
\end{definition}

\begin{corollary}[P3 as regulation-map constraint]\label{cor:p3_regulation}
Under any policy selecting from $R_{A}(\theta)$, property P3 (monotonic restriction, Definition~\ref{def:avf}) holds automatically: every admissible control satisfies $u' \leq u_{\mathrm{current}}$ or $u' = \mathrm{STOP}$, so the viability space $V(u')$ parameterised by the new control is a subset of $V(u_{\mathrm{current}})$. Consequently, P3 is not an external axiom imposed on the Autopilot but the structural content of how $R_{A}$ is instantiated.
\end{corollary}

\paragraph{Trade-offs made explicit.} The formalisation surfaces three trade-offs that were implicit in v10. First, $R_{A}$ is \emph{irreversible}: once tightened, a component of $u$ cannot be relaxed inside $R_{A}$. Recovery from over-tightening requires operator action outside $R_{A}$---the ``P3 and over-restriction'' limitation of §\ref{sec:limitations}. Second, the floor $u_{\min}$ is a modelling choice: too tight a floor (close to $u_{\max}$) narrows the Autopilot's action space and forces frequent $\mathrm{STOP}$ fallbacks; too loose a floor (far from $u_{\max}$) permits over-restriction in pathological regimes. The operational defaults (section~\ref{sec:autopilot}) sit on the conservative side. Third, $\mathrm{STOP}$ is a hard fallback, not a continuation: the pathwise trajectory after $\mathrm{STOP}$ is not modelled; the system transitions to a human-in-the-loop recovery path. This is consistent with the fail-secure discipline of the pipeline but precludes fully autonomous long-horizon claims.

\subsection{Necessity Motivated by $\hat{B}(x)$ Structure}\label{sec:necessity}

The following necessity results are \emph{internal} to the $\hat{B}(x)$ framework: they establish that P1--P3 follow from the Informational Viability Principle as a design choice, not that they are universally required of any governance system. Concretely, every proposition below should be read as prefixed by ``for any governance system that explicitly separates observed capacity $S(x)$ from unobserved risk $\hat{B}(x)$ as independently estimable quantities''. Systems that do not make this separation (stateless policy engines, reactive-only safety filters, or probabilistic-guarantee schemes) address a different problem class; we do not claim they are inviable, only that they sit outside the scope of these results.

\begin{proposition}[Necessity of P1, within the $\hat{B}(x)$ framework]\label{prop:p1}
Any governance framework that estimates $\hat{B}(x) = U(x) + SB(x) + RG(x)$ and lacks continuous cross-request state cannot estimate the three terms of $\hat{B}(x)$ and therefore cannot apply the Informational Viability Principle.
\end{proposition}

\begin{proof}[Proof sketch]
Each term requires accumulated history. $U(x)$ requires a sliding window to detect distributional drift. $SB(x)$ requires accumulated segment counts to detect diverging block rates. $RG(x)$ requires session state to detect sequential patterns. Without P1, the system implicitly sets $\hat{B}(x) = 0$ for unobserved components, artificially inflating $\hat{\Phi}(x)$ until a violation occurs.
\end{proof}

\begin{proposition}[Necessity of P2, within the $\hat{B}(x)$ framework]\label{prop:p2}
Within the $\hat{B}(x)$ framework, any governance system that does not monitor $d\hat{B}/dt$ cannot prevent boundary crossings in domains with immediate real-world consequences.
\end{proposition}

\begin{proof}[Proof sketch]
The Informational Viability Principle requires $\hat{\Phi}(x) \geq \varepsilon$ at each decision point. In domains with immediate consequences (fund transfers, medical authorizations), a violation detected at the moment of occurrence has already caused harm. P2 requires estimating Aubin's exit function $\tau_{\mathrm{exit}}(\theta)$ of Definition~\ref{def:exit_function}---the first-passage time of the behavioral trajectory through $\partial V$: without it, the system cannot distinguish $VI = 0.6$ (stable) from $VI = 0.6$ (declining at 0.017/s; $\tau_{\mathrm{exit}} \approx 35$~s). Section~\ref{sec:tstar_exit} realises this estimation conservatively via the first-order approximation $t^* \approx \tau^{VI}_{\mathrm{exit}} - t_{\mathrm{now}} \geq \tau_{\mathrm{exit}} - t_{\mathrm{now}}$, where $\tau^{VI}_{\mathrm{exit}}$ is a $VI$-based scalar surrogate; tighter per-constraint detection is provided by the region classification of Definition~\ref{def:viability_regions}.
\end{proof}

\begin{proposition}[Necessity of P3, within the $\hat{B}(x)$ framework]\label{prop:p3}
Within the $\hat{B}(x)$ framework, any governance system without monotonic restriction allows adversarial manipulation to corrupt $\hat{B}(x)$ estimates, causing the system to believe $\hat{\Phi}(x) \geq \varepsilon$ when the actual margin is negative.
\end{proposition}

\noindent\textit{Argument.} P1 and P2 address the \emph{computability} of $\hat{B}(x)$. P3 addresses the \emph{integrity} of $\hat{B}(x)$ estimates under adversarial pressure---a distinct requirement. Under adversarial erosion (Category~II), an attacker can induce threshold relaxation ($\tau \uparrow$). When $\tau$ increases, $d_{C_1} = (\tau - KL)/\tau$ increases artificially---the system believes $\hat{\Phi}(x)$ is healthy when the agent has drifted. P3 eliminates this attack: interventions can only tighten~$\tau$, so $\hat{B}(x)$ estimates can only become more conservative. A formal treatment connecting P3 to adversarial robustness bounds for $\hat{B}(x)$ estimators remains an open direction for future work.

\subsection{Sufficiency and Constructive Mapping}

\begin{table}[ht]
\centering
\caption{Constructive mapping from $\hat{B}(x)$ terms to RiskGate estimators.}
\label{tab:constructive}
\small
\begin{tabular}{@{}llll@{}}
\toprule
\textbf{Property} & $\hat{B}(x)$ \textbf{Term} & \textbf{Category Countered} & \textbf{RiskGate Mechanism} \\
\midrule
P1 & $U(x)$ & Behavioral drift & Dirichlet + KL divergence \\
P1 & $SB(x)$ & Bias accumulation & BiasDetector $z$-test \\
P2 & $dU/dt$ & Calibration decay & D-UCB1 + Autopilot \\
P2 & $dRG/dt$ & Adversarial erosion & PlanGovernor multi-step detection \\
P3 & All terms & Systemic failure & Pipeline monotonicity (Proposition~\ref{thm:monotonicity}) \\
Infra & --- & Auditability & Structured tracing + Quality Evaluator \\
\bottomrule
\end{tabular}
\end{table}

\subsection{External Necessity: What Violation of Each P$_i$ Looks Like in Practice}

The necessity arguments of §\ref{sec:necessity} are internal to the $\hat{B}(x)$ framework. To avoid tautological reasoning we briefly note, for each property, a recognisable class of real-world governance systems that does \emph{not} implement it and the failure pattern that class exhibits. The point is not to disparage those systems---each has its own intended scope---but to show that absence of P$_i$ is not a purely internal stipulation: it corresponds to well-documented failure modes in the deployed-systems literature.

\paragraph{Ausence of P1 --- stateless policy engines.} Authorisation engines that evaluate each request in isolation (OPA, Cedar in its base form~\cite{cedar2024}) do not maintain cross-request state and therefore cannot express constraints that require aggregation over windows, segments, or sessions. The documented workaround is an external context provider, which is precisely what P1 prescribes. The corresponding failure class is compositional attacks that require sequence-level aggregation: the structuring pattern of §\ref{sec:demo_structuring} is the canonical example.

\paragraph{Absence of P2 --- reactive-only classifiers.} Supervised one-shot classifiers (gradient-boosted fraud detectors, static risk scorers) flag a decision only when the current observation crosses a threshold; they do not track $d\hat{B}/dt$ and have no lead time. The concept-drift detection literature~\cite{gama2014} has argued for anticipation mechanisms for over a decade. The failure class is any gradually-accumulating drift: the emergent-bias scenario of §\ref{sec:demo_bias} illustrates how a reactive $z$-test detects the problem 108~operations later than the Autopilot's continuous disparity monitoring, which acts on the same $SB(x)$ trajectory that the aggregate $t^*$ signals as declining ($b < 0$).

\paragraph{Absence of P3 --- override-enabled frameworks.} Governance systems with runtime \emph{break-glass} mechanisms (operator-accessible override flags, relax-on-demand policies) admit interventions that expand operational freedom, violating monotonic restriction. The failure class is compromise amplification: if the override channel is obtained by an adversary (credential theft, insider threat, misconfigured RBAC), defences remain nominally ``on'' while actually loosened. Enterprise security incident reports~\cite{verizondbir2024} document this pattern across multiple sectors. P3-compliant designs (tighten-only Autopilot, no runtime relax path) structurally prevent it.

These three patterns show that P1--P3 are not arbitrary framework choices: each responds to a distinct, empirically recurring failure class in governance systems that omit it. A longer treatment with full case studies is left for a follow-up work; here we flag the external anchoring as counterweight to the within-framework necessity arguments.

\subsection{Regulatory Alignment}

\begin{table}[ht]
\centering
\caption{Mapping of AVF properties to EU AI Act requirements.}
\label{tab:regulatory}
\small
\begin{tabular}{@{}lll@{}}
\toprule
\textbf{EU AI Act} & \textbf{AVF Property} & $\hat{B}(x)$ \textbf{Connection} \\
\midrule
Art.~9: Risk management & P1 (Monitoring) & Continuous estimation of all $\hat{B}(x)$ terms \\
Art.~15: Accuracy, robustness & P2 (Anticipation) & $d\hat{B}/dt$ monitoring and $t^*$ prediction \\
Art.~14: Human oversight & P3 + boundary zone & Monotonic restriction + escalation near $\partial V$ \\
\bottomrule
\end{tabular}
\end{table}

\section{Architecture}

RiskGate is organized as a three-layer directed acyclic graph (DAG): \texttt{governance/} $\leftarrow$ \texttt{domain/} $\leftarrow$ \texttt{agentcontrol/}. The governance layer implements statistical machinery for estimating $S(x)$ and $\hat{B}(x)$; the domain layer provides business rules and state providers that parameterize $RG(x)$; the abstraction layer offers a fluent policy API.

Segmentation keys enable per-population statistical tracking: $\mathrm{seg}_{\mathrm{TX}} = \text{tool name} \| \text{tier}$, $\mathrm{seg}_{\mathrm{Tools}} = \text{tool calls} \| \text{tier}$.

\section{Statistical Estimators}

All statistical mechanisms in RiskGate serve exactly one of two purposes: estimating the observed capacity $S(x)$, or estimating the unobserved risk bound $\hat{B}(x)$. This section presents them under that organizing principle, rather than as a flat enumeration of independent pillars.

\subsection{Estimating $S(x)$: Observed Capacity}

Two mechanisms contribute to the $S(x)$ estimate: the primary Bayesian behavioral profiler and the secondary quality evaluator.

\subsubsection{Bayesian Behavioral Profiling (Primary $S(x)$ Estimator)}

We model governance outcome frequencies as a multinomial distribution with Dirichlet conjugate prior. The posterior mean:
\begin{equation}
\mathbb{E}[\theta_k \mid \text{data}] = \frac{n_k + \alpha_k}{N + \sum_j \alpha_j}
\label{eq:dirichlet}
\end{equation}
gives $S(x)$: the posterior probability that the current action falls within the normal behavioral distribution. Empirical Bayes shrinkage pools segment-level and global estimates:
\begin{equation}
q_{\mathrm{mix}}(k) = \lambda \cdot q_{\mathrm{seg}}(k) + (1 - \lambda) \cdot q_{\mathrm{global}}(k), \quad \lambda = \frac{N_{\mathrm{seg}}}{N_{\mathrm{seg}} + \kappa}
\label{eq:shrinkage}
\end{equation}
with $\kappa = 500$ (TX channel) and $\kappa = 300$ (Tools channel). The shrinkage constant $\kappa$ sets the segment sample size at which $\lambda = 0.5$ (equal weight on segment and global); these values correspond to approximately $1.5{-}2\times$ the coherence window of each channel, so that segment-specific data is trusted equally to the global estimate only after roughly twice the window length has been observed for that segment. These are heuristic defaults chosen by design; cross-validation on deployment-specific traffic is recommended for production calibration.

\subsubsection{Quality Evaluator (Secondary $S(x)$ Estimator)}

Operates asynchronously, scoring five output-quality dimensions in $[0, 1]$. Provides a secondary $S(x)$ estimate for trend monitoring, independent of the Dirichlet-based primary estimator.

\subsection{Estimating $\hat{B}(x)$: Unobserved Risk}

Three statistically distinct components compose $\hat{B}(x)$, each requiring dedicated instruments.

\subsubsection{$U(x)$: Dual-Channel KL Divergence}

The $U(x)$ estimator uses forward KL divergence:
\begin{equation}
D_{\mathrm{KL}}(p \| q) = \sum_k p(k) \log_2 \frac{p(k)}{q(k)} \quad [\text{bits}]
\label{eq:kl}
\end{equation}

We monitor two independent channels: Channel~A (transaction outcomes, window $W_A = 300$) and Channel~B (tool calls, window $W_B = 200$). This dual-channel architecture ensures that adversarial manipulation of one channel cannot mask drift in the other.

The KL threshold is calibrated via Monte Carlo to control the per-channel false positive rate at $\alpha_{\text{ch}}$:
\begin{equation}
\tau = Q_{1-\alpha_{\text{ch}}}\!\left\{D_{\mathrm{KL}}(\hat{p}_i \| q)\right\}_{i=1}^{5000}, \quad \hat{p}_i = X_i / W, \quad X_i \sim \mathrm{Multinomial}(W, q)
\label{eq:threshold_calibration}
\end{equation}

\paragraph{Joint FPR control across channels (Bonferroni).} Eq.~(\ref{eq:threshold_calibration}) controls the false positive rate on a \emph{single} channel. Because the engine raises an alarm whenever either channel exceeds its threshold---a logical disjunction $\{KL_A > \tau_A\} \lor \{KL_B > \tau_B\}$---the joint FPR under independence is $1 - (1 - \alpha_{\text{ch}})^2 \approx 2\alpha_{\text{ch}}$, and can be larger if the channels exhibit positive dependence. To bound the joint FPR at a target family-wise level $\alpha$, we apply a Bonferroni correction and calibrate each channel at
\begin{equation}
\alpha_{\text{ch}} = \alpha / K
\label{eq:bonferroni}
\end{equation}
where $K = 2$ is the number of channels. Substituting $\alpha_{\text{ch}}$ into eq.~(\ref{eq:threshold_calibration}) yields thresholds satisfying $\Pr(\text{false alarm} \mid H_0) \leq \alpha$ by the Bonferroni inequality, regardless of dependence between channels. Default deployment uses $\alpha = 0.01$, so $\alpha_{\text{ch}} = 0.005$ and the quantile in eq.~(\ref{eq:threshold_calibration}) becomes $Q_{0.995}$. The cost is a modest reduction in per-channel detection power; the benefit is a provable family-wise guarantee that is honest about the multiple-comparison structure of the detector. For $K > 2$ (e.g.\ future addition of latency or cost channels) Šidák's correction $\alpha_{\text{ch}} = 1 - (1 - \alpha)^{1/K}$ is slightly more powerful while retaining the joint guarantee under independence.

\subsubsection{Adaptive $U(x)$: D-UCB1 Threshold Optimization}

Fixed KL thresholds are suboptimal in non-stationary environments. We treat threshold selection as a multi-armed bandit, maintaining 4 independent bandits $\{\mathrm{TX}, \mathrm{Tools}\} \times \{\text{standard}, \text{gold}\}$ with 7--8 arms each. Discounted UCB1 selects:
\begin{equation}
a^* = \arg\max_a \left[\bar{Q}_\gamma(a) + c \sqrt{\frac{\ln W_{\text{total}}}{W(a)}}\right], \quad c = \sqrt{2}
\label{eq:ucb1}
\end{equation}
with decay $\gamma = 0.999$. The reward matrix penalizes false positives more heavily than false negatives ($r_{FP} = -2.0 > r_{FN} = -1.0$), reflecting operational asymmetry.

\paragraph{On the exploration constant $c$.} We use $c = \sqrt{2}$, the constant of the stationary UCB1~\cite{auer2002}. Garivier and Moulines~\cite{garivier2011} show that the $O(\sqrt{\Gamma_T T \log T})$ regret order of Proposition~\ref{prop:regret} holds for any exploration constant above a finite threshold; the constant multiplier, however, depends on the discount factor $\gamma$, on the breakpoint cadence $\Gamma_T$, and on the effective sample size $W(a)$ under discounting. $c = \sqrt{2}$ is not claimed to be optimal for the discounted variant: we retain it as a conservative default inherited from the stationary case, and flag deployment-specific tuning of $c$ against observed $\gamma$ and $\Gamma_T$ as a natural extension.

\begin{proposition}[Regret regime]\label{prop:regret}
The D-UCB1 regret bound
\begin{equation}
\mathbb{E}[R_T] = O\!\left(\sqrt{\Gamma_T \cdot T \cdot \ln T}\right)
\end{equation}
of~\cite{garivier2011}, where $\Gamma_T$ is the number of change-points over $T$ rounds, \emph{applies under the conditions of}~\cite{garivier2011}, which this architecture satisfies for assumptions~(i)--(iii) by construction and assumes for~(iv).
\end{proposition}

\paragraph{Verification of the bound's hypotheses.} We make the assumptions of~\cite{garivier2011} explicit and indicate how each is met. Table~\ref{tab:reward-matrix} collects the raw reward values used in Eq.~\ref{eq:ucb1} together with their $[0,1]$-normalised counterparts, so that the scales referenced by hypotheses~(i)--(iii) are visible in one place.

\begin{table}[ht]
\centering
\caption{D-UCB1 reward matrix. Raw rewards $r$ encode operational asymmetry ($r_{FP} < r_{FN}$); normalised rewards $\tilde{r} = (r - r_{\min})/(r_{\max} - r_{\min}) \in [0, 1]$, with $r_{\min} = -2.0$ and $r_{\max} = 1.5$, satisfy hypothesis~(i) of Proposition~\ref{prop:regret} and are the values fed to UCB1.}
\label{tab:reward-matrix}
\small
\begin{tabular}{@{}lccl@{}}
\toprule
\textbf{Outcome} & \textbf{Raw } $r$ & \textbf{Normalised } $\tilde{r}$ & \textbf{Interpretation} \\
\midrule
True positive  ($r_{TP}$) & $+1.5$ & $1.000$ & Correct block / flag \\
True negative  ($r_{TN}$) & $+0.5$ & $0.714$ & Correct allow \\
False negative ($r_{FN}$) & $-1.0$ & $0.286$ & Missed anomaly \\
False positive ($r_{FP}$) & $-2.0$ & $0.000$ & Spurious block (most penalised) \\
\bottomrule
\end{tabular}
\end{table}

\begin{description}
\item[(i) Bounded rewards in {$[0, 1]$}.] The raw reward matrix (Table~\ref{tab:reward-matrix}) is mapped to $[0, 1]$ by an affine normaliser $\tilde{r} = (r - r_{\min}) / (r_{\max} - r_{\min})$ with $r_{\min} = -2.0$ and $r_{\max} = 1.5$ before being fed to UCB1. All bandit statistics ($\bar{Q}_\gamma(a)$, the upper-confidence term) operate on $\tilde{r}$. The bound therefore applies to the normalised index; the asymmetric operational scale is preserved only for external consumers (attribution, telemetry) and does not enter the regret analysis.
\item[(ii) Fixed discount factor $\gamma \in (0, 1)$.] $\gamma = 0.999$ is fixed at configuration time and does not depend on the arm, the segment, or the round. The global-decay update (every arm decays at every step) matches the standard D-UCB formulation, not the per-pull decay of some variants.
\item[(iii) Exploration constant.] The exploration constant $c = \sqrt{2}$ and the use of discounted counts $W(a)$ / $W_{\text{total}}$ in the UCB term follow the definition of~\cite{garivier2011} verbatim.
\item[(iv) Piecewise-stationary environment with finite $\Gamma_T$.] This is an assumption on the environment, not on the algorithm. We motivate it empirically: observed breakpoints in production (profile changes, seasonal shifts, adversarial campaigns) arrive at a cadence for which $\Gamma_T$ grows sub-linearly in $T$. We do not \emph{prove} stationarity; we assume it at the regime level. Violations of intra-segment i.i.d.\ behaviour (e.g.\ diurnal autocorrelation, bursty adversarial traffic) do not break the asymptotic order of the bound but degrade the hidden constants. This is a limitation shared by all regret guarantees for bandits deployed on real user streams; we flag it rather than paper over it.
\item[(v) Delayed feedback.] In deployment the reward is observed with lag: TP/FP/TN/FN labels are derived from post-hoc ground truth (fraud investigation, regulatory adjudication, analyst review) that arrives hours to months after the decision. \cite{garivier2011} assumes immediate feedback. We assume the lag distribution is stationary and bounded within each segment, which preserves the $O(\sqrt{\Gamma_T T \ln T})$ regret order at the cost of a constant factor that grows with the maximum lag. Lag distributions that are themselves non-stationary (e.g.\ investigation backlogs during incident spikes) violate this weaker assumption and are not covered by the bound; we flag this as an open limitation rather than claim it away.
\end{description}
\noindent The normalisation in~(i) changes the gap $\Delta_a$ between arms by the same constant factor, which affects the multiplicative constant of the bound but not its $\sqrt{\Gamma_T T \ln T}$ order. The four bandits $\{\mathrm{TX}, \mathrm{Tools}\} \times \{\text{standard}, \text{gold}\}$ are optimised independently; Proposition~\ref{prop:regret} applies per bandit. A joint bound across the four instances is not claimed.

\subsubsection{$U(x)$ Stability: Emergency Kill Switch}

When the D-UCB1 bandit shows signs of instability (distribution shift, data poisoning), the kill switch reverts to conservative static thresholds. The combined stability score:
\begin{equation}
S_{\mathrm{stab}} = 0.6 \cdot z_r + 0.4 \cdot z_s
\end{equation}
triggers at $S_{\mathrm{stab}} > \tau_{\mathrm{ks}} = 3.0$. The weights $(0.6, 0.4)$ sum to one so that $S_{\mathrm{stab}}$ is expressed on the same $\sigma$ scale as the individual z-scores, rendering the trigger $\tau_{\mathrm{ks}} = 3.0$ interpretable as approximately a combined three-sigma deviation above baseline. The $1.5\times$ preference for $z_r$ (reward-variance z-score) over $z_s$ (switch-rate z-score) reflects that reward volatility is a direct indicator of system instability---measurable in observed outcomes---while switch rate is a meta-indicator of the learner's inability to settle, and is therefore given slightly lower weight when the two signals disagree. These are design-chosen priors; deployment-specific tuning of the weight ratio is recommended when baseline data is available.

\subsubsection{$SB(x)$: Algorithmic Bias Detection}

The $SB(x)$ estimator uses a segment-vs-rest $z$-test:
\begin{equation}
Z = \frac{|p_s - p_r|}{\sqrt{\hat{p}(1 - \hat{p})(1/n_s + 1/n_r)}}
\label{eq:ztest}
\end{equation}
This eliminates the self-inclusion bias of naive segment-vs-global comparisons. An alert fires when both $|p_s - p_r| > 0.25$ and $p$-value $< 0.05$.

\paragraph{Assumptions and statistical caveats.} Eq.~(\ref{eq:ztest}) is the standard two-proportion $z$-test and inherits its classical assumptions, two of which deserve explicit discussion in the context of a sequential agent:

\emph{(a) Independence across observations.} The $z$-test assumes the $n_s$ and $n_r$ Bernoulli outcomes (block / allow) are i.i.d. In production the stream is autocorrelated: repeat customers, diurnal patterns, bursty adversarial campaigns, and tier-level seasonality all induce positive serial dependence. Under positive autocorrelation the nominal variance $\hat{p}(1-\hat{p})/n$ understates the true variance of the contrast, so $Z$ is inflated and the false positive rate exceeds the nominal $\alpha$. A conservative adjustment replaces each $n_i$ with an effective sample size
\begin{equation}
n_{\text{eff}} = n \cdot \frac{1 - \rho_1}{1 + \rho_1}
\label{eq:ess}
\end{equation}
where $\rho_1$ is the lag-1 autocorrelation of the segment's block-rate series, clamped to $[0, \rho_{\max}]$ for numerical stability. Substituting $n_{\text{eff},s}$ and $n_{\text{eff},r}$ into eq.~(\ref{eq:ztest}) preserves the statistic's form while restoring nominal-level FPR under serial dependence. This correction is opt-in: deployments prioritising detection power may leave it disabled and accept that the nominal $\alpha$ is an optimistic estimate of the true FPR.

\emph{(b) Multiple testing across segments.} The detector evaluates $S$ segments simultaneously against the rest; under independence the family-wise false positive rate is $1 - (1 - \alpha)^S$, which for $S = 5$ and $\alpha = 0.05$ already exceeds $0.22$. To bound the family-wise error at target level $\alpha_{\text{fwer}}$, we apply a Bonferroni correction to the per-segment threshold:
\begin{equation}
\alpha_{\text{seg}} = \alpha_{\text{fwer}} / S
\label{eq:bonferroni_seg}
\end{equation}
with $S$ taken as the number of \emph{active} segments in the current window. For the default $\alpha_{\text{fwer}} = 0.05$ and $S = 5$, this hardens the per-segment threshold to $\alpha_{\text{seg}} = 0.01$. Holm--Bonferroni is a slightly more powerful alternative under positive dependence between segments (common when segmentations overlap, e.g.\ tier $\cap$ region), and is interchangeable with Bonferroni in the same pipeline.

Both corrections are low-cost, architecture-preserving, and composed as follows with the KL family-wise correction of eq.~(\ref{eq:bonferroni}): each detector (KL, bias, kill switch) is treated as its own family with its own FWER budget; we do \emph{not} claim a joint FWER across detector families. Doing so would require a pipeline-level budget allocation (see section~\ref{sec:limitations}); we flag this as a limitation rather than paper over it. Higher-order dependence structure (multi-lag autocorrelation, overlapping segmentations) is not corrected by eqs.~(\ref{eq:ess})--(\ref{eq:bonferroni_seg}); block bootstrap on the Bernoulli series is the rigorous alternative and is left to future work.

\subsection{Infrastructure}

\paragraph{Structured Tracing.} Every subsystem decision is emitted as a structured JSONL event with ten event types, wrapped in try/except to ensure logging failures never affect governance decisions. This is an infrastructure component that does not directly estimate $S(x)$ or $\hat{B}(x)$ but supports auditability (Table~\ref{tab:constructive}, ``Infra'' row).

\section{Integrated Pipeline and Formal Properties}\label{sec:pipeline}

\subsection{Pipeline Structure}

The \texttt{ENGINE.evaluate()} pipeline comprises stages that fall into two categories: \textbf{decisory stages} that participate in the final allow/block determination, and \textbf{state-update stages} that maintain estimator accuracy without modifying the decision.

\paragraph{Decisory stages} (participate in $D_{\text{final}}$):
\begin{description}[leftmargin=1.5em]
\item[S1] Regulatory sandbox check (EU AI Act Art.~57--62).
\item[S2] Failsafe: if kill switch active, block all non-safe tools [$U(x)$ stability].
\item[S3] Stateful domain rule hooks: YAML-declared rules requiring domain state [$RG(x)$ from domain rules].
\item[S8] Drift verification: if $KL > \tau$, block [$\hat{\Phi}(x) < \varepsilon$].
\end{description}

\paragraph{State-update stages} (maintain estimator accuracy, do not modify $D$):
S4 (segmentation), S5 (Bayesian learner updates for $S(x)$), S6 (UCB1 threshold selection for $U(x)$), S7 (KL computation on both channels), S10 (bandit feedback for $U(x)$ threshold learning), and S11 (AVF viability snapshot: compute $VI(t)$ operationalizing $\hat{\Phi}(x)$, classify region, predict $t^*$; observability only).

\subsection{Pipeline Monotonicity}

\begin{proposition}[Pipeline monotonicity, by construction]\label{thm:monotonicity}
Each decisory stage~$i$ can only set $D \leftarrow \text{block}$; no stage can set $D \leftarrow \text{allow}$ once $D = \text{block}$:
\begin{equation}
D_{\text{final}} = D_{S1} \wedge D_{S2} \wedge D_{S3} \wedge D_{S8} \wedge D_{S9}
\label{eq:monotonicity}
\end{equation}
\end{proposition}

\begin{proof}[Verification]
The property holds by construction under sequential evaluation. S1--S2 (sandbox, kill switch) and S3 (domain rules) perform early return with $D = \text{block}$; subsequent stages are never reached. S8 (drift) can change $D$ from \emph{allow} to \emph{block} but never the reverse. State-update stages (S4--S7, S10--S11) modify estimator state but not the decision. A strengthened result establishing monotonicity under concurrent evaluation of decisory stages and under partial stage failures (e.g.\ a stage raising an exception) is left as future work; see section~\ref{sec:limitations}.
\end{proof}

This property ensures P3 is satisfied at the pipeline level by construction: $\hat{B}(x)$ estimates can only trigger restrictions, and restrictions can never be reversed within a single evaluation.

\subsection{Operational Argument for P3 under Concurrency}\label{sec:concurrency_argument}

Proposition~\ref{thm:monotonicity} is a statement about one execution of the pipeline. The reference implementation serves many evaluations concurrently; we argue informally that P3 is preserved in practice under three observations about how the implementation is structured.

\paragraph{Observation 1 --- Concurrency is between evaluations, not within.} Each call to \texttt{ENGINE.evaluate()} executes its decisory stages sequentially on a single thread; the decision variable $D$ is a local of that call and is never mutated by another evaluation. Multiple concurrent calls do not interleave stages within a single evaluation. Proposition~\ref{thm:monotonicity} applies to each call individually; concurrent calls do not share $D$ and therefore cannot compose into a non-monotonic aggregate.

\paragraph{Observation 2 --- Shared estimator state is guarded.} The state read during evaluation (Dirichlet counts, bandit arms, kill-switch baselines, SVP history) lives in singletons protected by \texttt{threading.Lock}. Reads and writes are atomic at the granularity required by each stage. Races on counters introduce statistical noise in subsequent estimates ($\hat{B}(x)$ slightly less accurate) but do not produce a decision that relaxes a prior restriction within the same evaluation: the decision is computed from the locally read state and committed as a conjunction with prior stages.

\paragraph{Observation 3 --- Interventions are monotonic writers.} The Autopilot runs in a separate thread and can tighten thresholds while the engine is mid-evaluation. Its update contract is tighten-only (section~\ref{sec:autopilot}): it can never loosen $\tau$ or relax a constraint. An evaluation that reads a post-intervention threshold sees a value no looser than a pre-intervention read; one that reads a pre-intervention threshold completes under the stricter prior regime. In neither ordering is $D$ relaxed by the concurrent write.

\paragraph{Residual gaps.} These observations are not a substitute for a formal proof. In particular, they do not cover: (i)~fail-stop behaviour of a stage that hangs without raising an exception (e.g.\ a blocked I/O call within a stateful domain rule), (ii)~profile changes applied mid-evaluation that alter the semantics of $\varepsilon$ between stages, and (iii)~adversarial schedulers that starve locks to produce pathological orderings. A formal treatment addressing each is scoped as future work (section~\ref{sec:limitations}). The operational claim is that, under the locking and tighten-only discipline actually implemented, the regime in which Proposition~\ref{thm:monotonicity} fails is not exercised by the reference implementation, and any residual risk degrades $\hat{B}(x)$ accuracy rather than $D$ monotonicity.

\section{Multi-Step Governance: $RG(x)$ via PlanGovernor}

The per-operation engine estimates $U(x)$ and $SB(x)$ but cannot detect the $RG(x)$ component: fraud patterns that emerge only from the composition of individually legitimate operations.

Four canonical $RG(x)$ patterns:
\begin{itemize}
\item \textbf{Structuring:} $n \geq 3$ transfers to the same destination, each below $\tau_{\text{dual}}$, but summing to $\geq \tau_{\text{dual}}$.
\item \textbf{Account takeover:} Credential-change ticket combined with fund transfer in the same plan.
\item \textbf{Probing:} Strictly increasing amounts to the same destination.
\item \textbf{Round-tripping:} Cycles in the directed transfer graph, detected via DFS.
\end{itemize}

Level~1 (PlanGovernor) and Level~2 (per-step engine) analyze disjoint features:
\begin{equation}
D_{\text{final}}(P) = D_{\text{plan}}(P) \wedge \bigwedge_{i=1}^{n} D_{\text{step}}(s_i)
\label{eq:plan_decision}
\end{equation}

\section{Closed-Loop Autopilot}\label{sec:autopilot}

The Autopilot is the operational instance of Aubin's regulation map $R_{A}(\theta)$ formalised in §\ref{sec:autopilot_regulation} (Definition~\ref{def:ra_policy}): it selects controls from the tighten-only admissible set $R_{A}(\theta) = \{u' \leq u_{\mathrm{current}}\} \cup \{\mathrm{STOP}\}$ of eq.~(\ref{eq:ra_def}), and by Theorem~\ref{thm:viab_ra_nonempty} the regulated viability kernel $\mathrm{Viab}_{R_{A}}(V^{\dagger})$ is non-empty by construction. Concretely, the Autopilot maintains $\hat{\Phi}(x) = S(x) - \hat{B}(x) \geq \varepsilon$ autonomously by adjusting $\hat{B}(x)$ estimation parameters within $R_{A}(\theta)$. Targets requiring tightening (under-detection of growing $\hat{B}(x)$) are corrected automatically; targets requiring relaxation fall outside $R_{A}(\theta)$ by construction and are flagged for human intervention, preserving P3 as the structural content of Corollary~\ref{cor:p3_regulation}.

Three viability-specific triggers activate immediate tightening---i.e., select $u' < u_{\mathrm{current}}$ from $R_{A}(\theta)$: (1)~$VI(t) < 0.3$; (2)~$t^* < t^*_{\text{critical}}$; (3)~$\mathrm{region}_{\text{global}} = \text{exterior}$. When tightening would drive a binding-constraint component below $u_{\min}$ at $\theta \in \partial V$, the policy selects $\mathrm{STOP}$ and the system transitions to the controlled-exit subset of $V^{\dagger}$.

\section{Analytical Validation}\label{sec:analytical_validation}

This section provides analytical coverage arguments: mappings of the $\hat{B}(x)$ framework onto prior empirical failure taxonomies. Worked examples with measured outcomes are outside the scope declared in §\ref{sec:limitations}.

\subsection{Against Agents of Chaos}

Shapira et al.~\cite{shapira2026} documented eleven failure case studies. Table~\ref{tab:agents_chaos} maps each through the $\hat{B}(x)$ lens.

\begin{table}[ht]
\centering
\caption{Coverage mapping of Agents of Chaos failure cases against the $\hat{B}(x)$ components that the RiskGate architecture engages for each.}
\label{tab:agents_chaos}
\small
\begin{tabular}{@{}llll@{}}
\toprule
\textbf{Case Study} & $\hat{B}(x)$ \textbf{Term} & \textbf{Coverage} & \textbf{Mechanism} \\
\midrule
\#1 Disproportionate response & $RG(x)$ & Partial & Domain hooks block destructive ops \\
\#2 Non-owner compliance & $RG(x)$ & Partial & PlanGovernor escalation sequences \\
\#3 Sensitive info disclosure & $RG(x)$ & Partial & Domain PII validators \\
\#4 Resource exhaustion & $U(x)$ & High & Kill switch detects $U(x)$ instability \\
\#5 Denial-of-service & $U(x)$ & High & $U(x)$ stability metrics \\
\#6 Provider values & $SB(x)$ & Partial & BiasDetector $z$-test \\
\#7 Agent harm (gaslighting) & $U(x)$ & Partial & KL drift detects behavioral shift \\
\#8 Identity spoofing & $U(x)$ + static & Partial & Post-compromise $U(x)$ growth \\
\#9 Agent collaboration & $RG(x)$ & Low & Per-agent silos \\
\#10 Agent corruption & $U(x)$ + $RG(x)$ & Partial & Post-corruption drift \\
\#11 Libelous broadcast & $RG(x)$ & Partial & PlanGovernor if tools governed \\
\bottomrule
\end{tabular}
\end{table}

\subsection{Convergence with Agent Drift}

Rath~\cite{rath2026} simulates 847 workflows and documents 42\% task success rate reduction from drift with median onset at 73 interactions (IQR: 52--114). Both frameworks independently select KL divergence for tool-usage drift detection, validating RiskGate's information-theoretic approach to $U(x)$ estimation. Agent Drift's ASI satisfies P1 (monitoring $U(x)$) but not P2 or P3---by the necessity propositions, this leaves the governance loop incomplete. Table~\ref{tab:drift} maps the mitigation strategies from Rath~\cite{rath2026} to $\hat{B}(x)$ terms and their RiskGate equivalents. According to Rath~\cite{rath2026}, the corresponding drift reductions are 51.9\%, 63.0\%, 70.4\%, and 81.5\% respectively for each strategy in isolation and combined; these figures are reproduced from~\cite{rath2026} and are not measurements of RiskGate.

\begin{table}[ht]
\centering
\caption{Agent Drift mitigation strategies and their $\hat{B}(x)$ interpretations. Drift-reduction percentages are reproduced from Rath~\cite{rath2026}; RiskGate Equivalent is our mapping.}
\label{tab:drift}
\small
\begin{tabular}{@{}llll@{}}
\toprule
\textbf{Strategy} & $\hat{B}(x)$ \textbf{Term} & \textbf{RiskGate Equivalent} & \textbf{Drift Reduction} \\
\midrule
Memory Consolidation & $U(x)$ & Dirichlet + sliding windows & 51.9\% \\
Drift-Aware Routing & $U(x)$ stability & Kill Switch + failsafe thresholds & 63.0\% \\
Adaptive Anchoring & $U(x)$ threshold & D-UCB1 + Bayesian references & 70.4\% \\
Combined & All $\hat{B}(x)$ & Full pipeline & 81.5\% \\
\bottomrule
\end{tabular}
\end{table}

\section{Related Work}

\paragraph{LLM Safety and Alignment.} Constitutional AI~\cite{bai2022} and RLHF~\cite{ouyang2022} address alignment at training time. RiskGate operates at deployment time, treating the LLM as a black box; a well-aligned model still requires runtime governance because alignment does not prevent $\hat{B}(x)$ from growing.

\paragraph{Guardrails and Output Filters.} NeMo Guardrails~\cite{rebedea2023} and Guardrails AI provide rule-based output filtering without distributional context, making them vulnerable to slow growth in $U(x)$ and $RG(x)$ attacks. They implement no $\hat{B}(x)$ estimation.

\paragraph{Drift Detection.} ADWIN~\cite{bifet2007} and Page-Hinkley tests are commonly used for concept drift on scalar statistics. RiskGate's dual-channel KL approach estimates $U(x)$ from categorical distributions with Dirichlet-smoothed references, which is more appropriate for governance metadata.

\paragraph{Algorithmic Fairness.} The segment-vs-rest $z$-test extends the standard two-proportion test~\cite{agresti2012} by correcting for self-inclusion bias---critical for accurate $SB(x)$ estimation in skewed populations typical of financial services.

\paragraph{Viability Theory.} Aubin's viability theory~\cite{aubin1991,aubin2011} has been applied to economics, robotics, and climate policy, but not---to our knowledge---to autonomous AI agent governance. The Informational Viability Principle provides the first formal connection between viability theory and the $\hat{B}(x)$ estimation framework.

\paragraph{Runtime Agent Governance.} Kaptein et al.~\cite{kaptein2026} formalize runtime governance with the execution path as the central object, defining compliance policies as deterministic functions $\pi(A, \Pi, s^*, \Sigma) \to [0,1]$ and a fleet-level objective $\mathbb{E}\!\left[\sum v_T\right] \leq B$. RiskGate takes the behavioral state $\theta(t)$ as the primary object instead, which admits formalization within Aubin's viability theory and supports predictive mechanisms ($t^*$) absent from path-based frameworks. Kaptein et al.\ identify behavioral drift as an open problem (\S6.1) and propose connecting fleet-level governance with per-agent guarantees as future work (\S6.2); the AVF addresses both through continuous $U(x)$ monitoring and Theorem~\ref{thm:viab_ra_nonempty}. The frameworks are complementary.

\section{Worked Demonstrations: End-to-End Mathematical Traces}\label{sec:worked_demonstrations}

Two complementary traces illustrate the structural role of each $\hat{B}(x)$ component. §11.1 exercises a structuring attack where $RG(x)$ is decisive and $VI(t)$ remains INTERIOR throughout --- confirming by design that plan-level threats bypass the scalar viability index (§3.6). §11.2 exercises an emergent-bias attack where $RG(x)$ and $U(x)$ remain silent, $SB(x)$ dominates, and $VI(t)$ degrades visibly through $d_{C_2}$, exercising $P_2$ anticipation before boundary crossing. Together the two traces make explicit the structural separation: $VI(t)$ operationalises $\hat{\Phi}(x)$ on the per-request components ($U$ and $SB$), while $RG(x)$ is resolved by the orthogonal plan-level gate. Both traces propagate the published equations under stated parameters, per the scope declared in §\ref{sec:limitations}.

\subsection{Structuring Attack: $RG(x)$ as Protagonist}
\label{sec:demo_structuring}

\subsubsection{Scenario}

An attacker executes five transfers of \$4,800 each to the same destination over 30~minutes. Each is below the dual-approval threshold $\tau_{\text{dual}} = \$5{,}000$; the aggregate \$24,000 constitutes structuring. A static governance system passes all five.

\subsubsection{Phase 0: Initial $\hat{B}(x)$ Baseline}

With $N = 100$ prior operations ($\mathbf{n} = (70, 20, 10)$ for balance/transfer/account-mgmt), the reference distribution is $\mathbf{q} \approx (0.689, 0.204, 0.107)$ and the calibrated threshold $\tau_{\text{tools}} = 0.085$~bits. Initial viability:
\begin{align}
d_{C_1} &= \min(0.995,\; 0.941) = 0.941 \label{eq:demo_dc1}\\[4pt]
VI(t_0) &= 0.5(0.941) + 0.3(0.920) + 0.2(0.833) = 0.913 \quad [\text{Interior}] \label{eq:demo_vi0}
\end{align}

\subsubsection{Phase I: Operations 1--4 ($U(x)$ Accumulation)}

KL divergence on Channel~B grows gradually but stays well below~$\tau$:
\begin{equation}
D_{\mathrm{KL}}(p \| q)\big|_{\mathrm{op.4}} = 0.00184 \text{ bits} \ll 0.085 \text{ bits}
\end{equation}

$VI$ remains in INTERIOR throughout, essentially unchanged from $VI(t_0) = 0.913$: the scalar viability index sees no threat. This is expected by design (§3.3): the danger is in $RG(x)$, which is a plan-level phenomenon invisible to the per-request components of $\hat{\Phi}(x)$. Static per-request evaluation fails for the same reason.

\subsubsection{Phase II: Operation 5 --- $RG(x)$ Convergence}

Before execution, PlanGovernor fires rule \texttt{pg\_structuring}:
\begin{itemize}
\item $n = 5$ transfers to same destination $\geq 3$: \checkmark
\item Each below $\tau_{\text{dual}}$: \checkmark
\item Aggregate \$24,000 $\gg$ \$5,000: \checkmark
\end{itemize}

PlanRiskScorer computes anomaly score $a = 0.96$, $s_{\text{final}} = 1.96 \geq 1.5 \Rightarrow$ blocked. The two-level decision structure makes the detection explicit. First, the per-request pipeline (Proposition~\ref{thm:monotonicity}) evaluates the fifth transfer in isolation:
\begin{equation}
D_{\text{step}} = \underbrace{D_{S1}}_{\text{ALLOW}} \wedge \underbrace{D_{S2}}_{\text{ALLOW}} \wedge \underbrace{D_{S3}}_{\text{ALLOW}} \wedge \underbrace{D_{S8}}_{\text{ALLOW}} \wedge \underbrace{D_{S9}}_{\text{ALLOW}} = \text{ALLOW}
\label{eq:demo_step}
\end{equation}
The per-request pipeline passes: the individual transfer is below threshold, KL drift is negligible, and no bias is detected. Then, the plan-level gate (section~7) evaluates the full sequence:
\begin{equation}
D_{\text{final}} = \underbrace{D_{\text{plan}}}_{\text{BLOCK}} \wedge \underbrace{D_{\text{step}}}_{\text{ALLOW}} = \text{BLOCK}
\label{eq:demo_final}
\end{equation}

\subsubsection{Phase III: Post-Decision Observability}

The block triggers cost attribution within the governance lineage: the stage classification (plan-level block via PlanGovernor) maps to a \texttt{risk\_multiplier} of $0.9$ on the blocked transaction value, and the foregone-fee component applies the configured \texttt{fp\_cost\_rate}. The conservative reward asymmetry of the D-UCB1 optimizer ($r_{FP} = -2.0$ vs.\ $r_{FN} = -1.0$) reflects the same principle that makes this block economically rational in regulated financial settings: the cost of missed fraud vastly exceeds the cost of blocked legitimate transactions.

\paragraph{Key insight.} No single $\hat{B}(x)$ term detected this attack in isolation. $U(x)$ (KL) produces only $\approx 0.002$~bits after 5~operations, two orders of magnitude below $\tau_{\text{tools}} = 0.085$; static rules pass each transfer; $VI$ remains INTERIOR throughout---confirming that the per-request components of $\hat{\Phi}(x)$ saw no threat. It is $RG(x)$, operating at plan-level through PlanGovernor, that provides the decisive block. The two-level architecture is not redundant: it is structurally necessary because $RG(x)$ manifests only in sequential composition, which no per-request estimator can capture. This is the constructive sufficiency argument of section~3.10 made concrete.

\subsection{Emergent Bias Attack: $SB(x)$ as Protagonist}
\label{sec:demo_bias}

The structuring trace of §\ref{sec:demo_structuring} exhibited $RG(x)$ as decisive while $VI(t)$ remained INTERIOR throughout. The complementary failure mode is one where $RG(x)$ and $U(x)$ remain silent and $SB(x)$ alone dominates. This subsection presents that case, exercising $P_2$ anticipation explicitly: $VI(t)$ degrades visibly through $d_{C_2}$, and the Autopilot intervenes \emph{before} the boundary is crossed.

\subsubsection{Scenario}

A credit pre-approval agent processes 500 loan applications. Minority segment share of traffic is $\pi_{\min} = 0.12$. No explicit rule references any demographic attribute. An emergent feedback loop develops: the D-UCB1 threshold optimizer converges to a decision boundary that, under covariate shift in the minority segment's feature distribution, produces a drifting disparity in block rates. Each individual decision passes the per-request pipeline --- static rules do not fire and $U(x)$ barely moves. The failure is visible only in the accumulated segment-vs-rest statistics.

Without intervention, block rates would evolve linearly from $p_{\min}(0) = p_{\mathrm{rest}}(0) = 0.12$ to $p_{\min}(500) = 0.38$ and $p_{\mathrm{rest}}(500) = 0.12$, producing an end-state disparity $\delta(500) = 0.26$, above the threshold $\delta_{\text{bias}} = 0.25$ of the production profile.

\subsubsection{Phase 0: Initial $\hat{B}(x)$ Baseline (op 50)}

With sample disparity $\delta(50) \approx 0.026$, KL divergence $D_{\mathrm{KL}} \approx 0.003$~bits, and stable bandit variance $\sigma \approx 0.5$:
\begin{align}
d_{C_1}(50) &= (0.085 - 0.003)/0.085 = 0.965 \label{eq:bias_dc1_0}\\
d_{C_2}(50) &= (0.25 - 0.026)/0.25 = 0.896 \label{eq:bias_dc2_0}\\
d_{C_3}(50) &= (3.0 - 0.5)/3.0 = 0.833 \label{eq:bias_dc3_0}\\
VI(50) &= 0.5(0.965) + 0.3(0.896) + 0.2(0.833) = 0.918 \quad [\text{Interior}] \label{eq:bias_vi_0}
\end{align}

\subsubsection{Phase I: Silent $SB(x)$ Accumulation (ops 50--300)}

The bandit's adaptive loop interacts with covariate shift, causing $\delta$ to grow at approximately $d\delta/dt \approx 5.2 \times 10^{-4}$ per operation. Table~\ref{tab:bias_trajectory} traces the viability components. Throughout this phase the bias detector cannot fire: below op~250, $n_{\min} < 30 = \texttt{min\_segment\_samples}$, and even with sufficient samples the cumulative $z$-statistic remains below the $p < 0.05$ threshold.

\begin{table}[h]
\centering
\small
\begin{tabular}{rccccccc}
\toprule
op & $\delta$ & $d_{C_1}$ & $d_{C_2}$ & $d_{C_3}$ & $VI$ & $z$ & $p$-value \\
\midrule
100 & 0.052 & 0.965 & 0.792 & 0.833 & 0.887 & --- & --- \\
200 & 0.104 & 0.962 & 0.584 & 0.833 & 0.823 & --- & --- \\
250 & 0.130 & 0.961 & 0.480 & 0.833 & 0.791 & 1.06 & 0.289 \\
300 & 0.156 & 0.960 & 0.376 & 0.833 & 0.760 & 1.43 & 0.152 \\
\bottomrule
\end{tabular}
\caption{Trajectory during silent $SB(x)$ accumulation. $d_{C_1}$ and $d_{C_3}$ stay essentially flat --- $U(x)$ and stability are not engaged --- while $d_{C_2}$ decays linearly. The $z$-test is not yet significant.}
\label{tab:bias_trajectory}
\end{table}

\subsubsection{Phase II: $P_1$/$P_2$ Monitoring and Autopilot Intervention (op 300)}

At op~300, the SVP tracker's OLS regression over the trajectory window ($W = 50$) of $VI$ values yields a slope
\begin{equation}
b = \frac{d\,VI}{dt} \approx w_2 \cdot \beta_{C_2} = 0.3 \times (-2.08 \times 10^{-3}) = -6.24 \times 10^{-4} \text{ per op}
\end{equation}
because only $d_{C_2}$ is declining ($d_{C_1}$ and $d_{C_3}$ contribute zero slope). Applying eq.~(\ref{eq:tstar}):
\begin{equation}
t^* = -\frac{VI(300)}{b} = -\frac{0.760}{-6.24 \times 10^{-4}} \approx 1219 \text{ ops} \label{eq:tstar_bias}
\end{equation}
The aggregate $t^*$ predicts $VI \to 0$ at op~$\approx 1519$---a distant horizon. This reflects the buffering effect of the weighted index: $C_1$ and $C_3$ hold $VI$ well above zero even as $C_2$ deteriorates. This is an inherent limitation of single-scalar anticipation for single-constraint failures (see §\ref{sec:limitations}).

However, $t^*$ is not the only signal. The SVP snapshot also identifies $C_2$\texttt{\_fairness} as the \emph{binding constraint} (smallest distance, $d_{C_2} = 0.376$), and the per-constraint trajectory projects $d_{C_2} < \varepsilon_{\text{fairness}} = 0.05$ (BOUNDARY entry) at op~$\approx 457$ and $d_{C_2} < 0$ (EXTERIOR crossing) at op~$\approx 481$. The structured log emits \texttt{emit\_svp\_viability} with \texttt{binding\_constraint = "fairness"}, the declining $VI$ trend ($b < 0$), and $t^* = 1219$.

The $t^*$-based trigger (trigger~2 of §\ref{sec:autopilot}) does not fire: $t^* = 1219 \gg t^*_{\text{critical}}$. Instead, the Autopilot acts on the \texttt{max\_segment\_disparity} business target. With $\delta(300) = 0.156$ exceeding the configured target threshold, the Autopilot triggers a tighten-only intervention (P3-compliant, Proposition~\ref{prop:p3}): the bandit's effective reward asymmetry $r_{FP}/r_{FN}$ is strengthened and disparity-aware penalization is activated. The intervention restricts the decision surface --- it cannot and does not relax any constraint --- and consequently flattens the disparity growth after op~300.

\subsubsection{Phase III: Counterfactual Projection (no Autopilot intervention)}

To make the value of continuous monitoring explicit, we project the trajectory absent the op-300 intervention and report the four events at which each signal would have fired (Table~\ref{tab:bias_counterfactual}).

\begin{table}[ht]
\centering
\caption{Counterfactual trajectory absent the Autopilot's op-300 intervention. Lead times are relative to op~300.}
\label{tab:bias_counterfactual}
\small
\begin{tabular}{@{}llc@{}}
\toprule
Event & Op & Lead time \\
\midrule
Autopilot disparity-target intervention (actual) & 300 & --- \\
Reactive $z$-test crosses $p < 0.05$ ($\hat p_{\min}=0.226$, $\hat p_{\mathrm{rest}}=0.120$, $z \approx 2.05$) & $\approx 408$ & $+108$ \\
BOUNDARY entry ($d_{C_2} = \varepsilon_{\mathrm{fairness}} = 0.05$) & $\approx 457$ & $+157$ \\
EXTERIOR crossing ($d_{C_2} = 0$) & $\approx 481$ & $+181$ \\
\bottomrule
\end{tabular}
\end{table}

\noindent By op~500 the viability components would read $d_{C_2}(500) = -0.040$ and $VI(500) \approx 0.634$: the region is classified EXTERIOR by the worst-constraint rule (§3.3) because $d_{C_2} < 0$, even though $VI > 0$---the pipeline acts on the binding-constraint signal, not on the sign of $VI$. The cumulative disparity $\delta(500) = 0.26$ would exceed both the production threshold and the disparate-impact ``four-fifths'' ratio of 0.80 (ECOA/CFPB), exposing the operator to precisely the regulatory liability that bias-detection governance is intended to prevent.

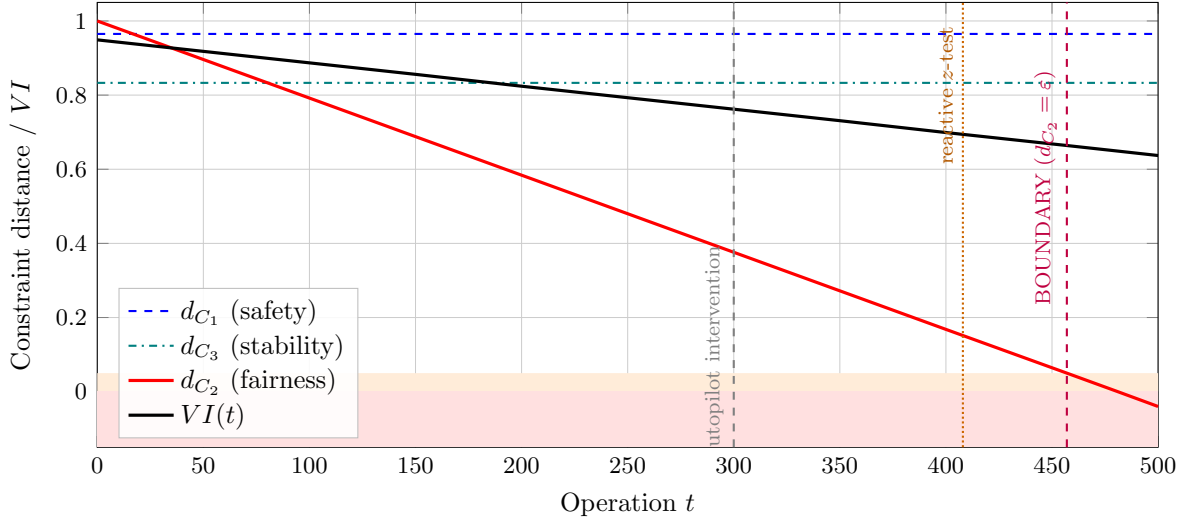
\begin{figure}[h]
\centering
\begin{tikzpicture}
\begin{axis}[
    width=0.95\linewidth,
    height=7.5cm,
    xlabel={Operation $t$},
    ylabel={Constraint distance / $VI$},
    xmin=0, xmax=500,
    ymin=-0.15, ymax=1.05,
    grid=both,
    grid style={line width=.1pt, draw=gray!20},
    major grid style={line width=.2pt, draw=gray!40},
    legend style={at={(0.02,0.02)}, anchor=south west, font=\small, draw=gray!50, fill=white, fill opacity=0.9, text opacity=1},
    legend cell align={left},
    tick label style={font=\footnotesize},
    label style={font=\small},
]
\addplot[draw=none, fill=orange!15, forget plot] coordinates {(0,0) (500,0) (500,0.05) (0,0.05)} \closedcycle;
\addplot[draw=none, fill=red!12, forget plot] coordinates {(0,-0.15) (500,-0.15) (500,0) (0,0)} \closedcycle;
\addplot[blue, thick, dashed] coordinates {(0,0.965) (500,0.965)};
\addlegendentry{$d_{C_1}$ (safety)}
\addplot[teal, thick, dashdotted] coordinates {(0,0.833) (500,0.833)};
\addlegendentry{$d_{C_3}$ (stability)}
\addplot[red, very thick] coordinates {(0,1.000) (50,0.896) (100,0.792) (150,0.688) (200,0.584) (250,0.480) (300,0.376) (350,0.272) (400,0.168) (457,0.050) (481,0.000) (500,-0.040)};
\addlegendentry{$d_{C_2}$ (fairness)}
\addplot[black, very thick] coordinates {(0,0.949) (50,0.918) (100,0.887) (150,0.856) (200,0.824) (250,0.793) (300,0.762) (350,0.731) (400,0.699) (457,0.664) (481,0.649) (500,0.637)};
\addlegendentry{$VI(t)$}
\addplot[gray, dashed, thick, forget plot] coordinates {(300,-0.15) (300,1.05)};
\node[anchor=south, font=\scriptsize, rotate=90, gray] at (axis cs:300,0.10) {Autopilot intervention};
\addplot[purple, dashed, thick, forget plot] coordinates {(457,-0.15) (457,1.05)};
\node[anchor=south, font=\scriptsize, rotate=90, purple] at (axis cs:457,0.55) {BOUNDARY ($d_{C_2} = \varepsilon$)};
\addplot[orange!80!black, densely dotted, thick, forget plot] coordinates {(408,-0.15) (408,1.05)};
\node[anchor=south, font=\scriptsize, rotate=90, orange!70!black] at (axis cs:408,0.80) {reactive $z$-test};
\end{axis}
\end{tikzpicture}
\caption{Viability trajectory under the emergent-bias scenario of §\ref{sec:demo_bias}. $d_{C_1}$ and $d_{C_3}$ remain flat while $d_{C_2}$ decays linearly, dragging $VI(t)$ down. The orange band $[0, 0.05]$ marks the BOUNDARY region; the red band $d_{C_2} < 0$ marks the EXTERIOR. The Autopilot's disparity-target intervention at op~300 acts 157 operations before BOUNDARY entry (op~457) and 108 operations before the reactive $z$-test would have fired (op~408). The aggregate $t^* = 1219$~ops (eq.~\ref{eq:tstar_bias}) predicts $VI \to 0$ at op~$\approx 1519$ --- a distant horizon buffered by healthy $C_1$ and $C_3$.}
\label{fig:bias_trajectory}
\end{figure}

\paragraph{Key insight.} No plan-level pattern detected this attack; no single request violated any rule; $U(x)$ never moved. The block-rate disparity grew continuously as a byproduct of adaptive threshold optimization interacting with covariate shift in the minority segment --- an emergent property of the learning loop itself. It is $SB(x)$, the per-request $z$-test on segment-vs-rest proportions, that \emph{would have} detected the failure reactively; it is the Autopilot's continuous monitoring of the disparity trajectory (a $P_1$ mechanism) that provided 108 operations of lead time before the reactive $z$-test would have fired. The aggregate $t^*$ (eq.~\ref{eq:tstar_bias}) signalled ongoing degradation ($b < 0$) and the SVP tracker identified the binding constraint ($C_2$), but the scalar $t^*$ gave a distant horizon because $C_1$ and $C_3$ buffered $VI$. This illustrates a structural limitation: single-scalar anticipation from $VI(t)$ is inherently conservative for single-constraint failure modes. Together with §\ref{sec:demo_structuring}, this counterfactual establishes the empirical separation claimed in §3.6: $VI(t) \approx \hat{\Phi}(x)$ operates correctly on the per-request components ($U$ and $SB$) and supports trend monitoring via $b < 0$ and binding-constraint identification, while $RG(x)$ requires the orthogonal plan-level gate. The two demonstrations are complementary by construction --- structuring exercises $RG \to D_{\text{plan}}$; emergent bias exercises $SB \to d_{C_2} \to$ Autopilot disparity target --- and neither mechanism could have prevented the other's failure.

\section{Discussion and Limitations}\label{sec:limitations}

\paragraph{Scope of the contributions.} We state the overall scope once here and do not repeat it at every figure, table, and derivation. The contributions of this paper are a framework (Agent Viability Framework with P1--P3), a theoretical organising principle (Informational Viability Principle), a reference implementation (RiskGate with complete statistical estimators), and analytical evidence of coverage (§\ref{sec:analytical_validation}, §\ref{sec:worked_demonstrations}). The contributions are \emph{not} a quantitative empirical evaluation: we do not report measured detection rates, false positive rates, latencies, monetary ROI, or head-to-head comparisons against baselines on live or synthetic traces. The worked demonstrations trace equations under assumed parameters; the coverage table is an analytical mapping; the ROI is absent by choice. A calibrated empirical evaluation against baselines is scoped for subsequent work. The remaining paragraphs of this section expand on specific limitations within this scope.

\paragraph{Empirical evaluation.} A measured evaluation---synthetic generators for each of the four degradation categories, detection-rate / false-positive-rate / latency curves with confidence intervals, ablations of $U(x)$, $SB(x)$, and $RG(x)$, and head-to-head comparison against stateless policy engines (e.g.\ OPA) and a supervised baseline---is the natural next iteration.

\paragraph{Formal first-passage-time analysis.} Thesis~\ref{thesis:senescence} (bounded senescence) is anchored empirically in observed drift rates~\cite{shapira2026,rath2026} rather than derived from a stochastic model. A formal treatment---choosing an explicit dynamics for $\theta(t)$ (drift + diffusion, Ornstein--Uhlenbeck, or piecewise-stationary jumps), deriving $E[T^*]$ as a function of drift rate, noise variance, and safety margin $\varepsilon$, and calibrating the model against the $73$-interaction median onset of~\cite{rath2026}---would strengthen the thesis from an empirical bound to a model-based prediction. We scope it as future work rather than speculate on the right dynamics without a measured evaluation to inform the choice.

\paragraph{Coverage table reproduction.} The reproduction of the eleven Agents of Chaos cases as deterministic attack generators in the harness---with measured TP/FP/FN, latency to first block, effect on $VI(t)$, and comparison against baseline policy engines---is the concrete next step. A representative subset (resource exhaustion, emergent bias, structuring-style tool misuse) would be a reasonable intermediate milestone.

\paragraph{Economic evaluation.} Monetary ROI figures are absent from the traces by design. A calibrated economic evaluation on real fraud/compliance traces is the appropriate context for such figures.

\paragraph{Monotonicity under concurrency and partial failures.} Proposition~\ref{thm:monotonicity} establishes pipeline monotonicity \emph{by construction} under sequential evaluation. A formal theorem establishing the same property under concurrent evaluation of decisory stages (interleavings over shared estimator state) and under partial stage failures (exceptions, timeouts) is not proved here. Both would require an explicit concurrency model and a formal treatment of fail-secure semantics; we leave them as future work.

\paragraph{Threshold sensitivity.} Hyperparameters ($\kappa$, KL thresholds, reward weights, $\tau_{\mathrm{ks}}$) are domain-dependent. The four built-in profiles and the Autopilot provide reasonable defaults, but initial calibration requires domain expertise.

\paragraph{Cold start.} With insufficient data, Dirichlet priors dominate and UCB1 bandits have not converged; $\hat{B}(x)$ estimates are imprecise. The calibration profile mitigates this.

\paragraph{Adversarial robustness.} A sophisticated adversary who understands the KL threshold structure could craft operations keeping $U(x)$ just below the drift threshold on both channels. Formal adversarial robustness guarantees for $\hat{B}(x)$ estimation remain an open problem. The KL partition itself (\texttt{tool\_name}$\times$\texttt{tier}, \texttt{tool\_calls}$\times$\texttt{tier}) is an operational choice rather than a derivation from first principles: Bonferroni protects the FPR \emph{given} the partition, not against a partition blind to the axis an adversary chooses to exploit (e.g.\ destinations, amounts, or recipients within a fixed tool type, which are invisible to the current grid). The $\hat{B}(x)$ decomposition mitigates this by composition rather than by completeness of any single detector---what $U(x)$ misses must surface in $SB(x)$ or $RG(x)$, as the structuring trace in §\ref{sec:demo_structuring} illustrates ($U(x)$ silent at two orders of magnitude below threshold; $RG(x)$ decisive). Low-cost extensions of $U(x)$ along axes the current partition does not cover are natural next channels and each is a $K{+}1$ addition in the Bonferroni step: (i)~\emph{latency per tool$\times$tier}, which would shift observably under resource exhaustion and denial-of-service---today only caught reactively by the kill switch; (ii)~\emph{cost and consumption} (tokens, nested calls), which catches runaway generation and prompt-injection patterns that induce extra work without being slow; and (iii)~\emph{outcome and error-code distributions}, under which an adversary probing permission boundaries produces a \texttt{permission\_denied} spike before finding a hole, yielding a distributional signature even under irregular probing. None of these closes the adversarial-bound gap, but each reduces the surface on which a partition-aware adversary can remain within $U(x)$ by construction.

\paragraph{$t^*$ limitations.} Three limitations of $t^*$ share a common root---it is a first-order estimator of an aggregate $VI$ trajectory---and all err conservatively. (i)~\emph{Linearity:} Proposition~\ref{prop:tstar_exit} assumes $v(t) = \mathbb{E}[VI(\theta(t))]$ is locally linear; under non-negligible $v''$ the boundary may be crossed before~$t^*$. Region classification is the instantaneous fallback, and the $\rho$ diagnostic (\S\ref{sec:inferential_profile}) warns when curvature becomes significant, but cannot correct the first-order estimate---higher-order and pathwise extensions via stochastic viability kernels~\cite{aubin2011} are flagged as natural follow-ups. (ii)~\emph{Aggregation:} because $t^*$ runs on the weighted average $VI$, it gives distant predictions when only one constraint deteriorates (see §\ref{sec:demo_bias}: $d_{C_2}$ reaches the boundary in $\approx 157$~ops while aggregate $t^* \approx 1219$~ops). Binding-constraint identification (P1) and per-metric business targets carry the actionable signal in this regime; a per-constraint $t^*$ variant fitting OLS on individual $d_{C_i}$ series would resolve this at the cost of three extra predictions per snapshot. (iii)~\emph{Intervention confounding:} $VI(t)$ depends on~$\tau$, so Autopilot tightening drops $d_{C_1}$ even when KL is stable, and the OLS cannot distinguish genuine drift from intervention artefacts; the resulting pessimistic bias is consistent with the fail-safe discipline.

\paragraph{P3: the weakest theoretical link.} Among the three AVF properties, P3 occupies a distinct epistemological position. P1 and P2 are necessary for \emph{computational} reasons: without accumulated cross-request state (P1), $\hat{B}(x)$ is undefined; without monitoring $d\hat{B}/dt$ (P2), violations in high-consequence domains cannot be prevented before they cause harm. These are necessity arguments grounded in the structure of $\hat{B}(x)$ estimation itself.

P3, by contrast, is necessary for \emph{integrity} reasons: it prevents adversarial manipulation of $\hat{B}(x)$ estimates via threshold relaxation. Proposition~\ref{prop:p3} identifies the specific attack vector (adversarial erosion causing $\tau \uparrow$, which inflates $d_{C_1}$ artificially) and shows P3 eliminates it by construction. However, the current argument establishes that P3 blocks one class of integrity attacks---it does not provide worst-case guarantees on $\hat{B}(x)$ estimation error under bounded adversarial perturbation. The gap between ``eliminates the known attack by construction'' and ``provides formal robustness bounds under an adversarial observation model'' remains the most significant open theoretical direction in this work. We flag this not as a weakness of the system---which is conservative by design---but as the frontier where the theoretical contribution requires extension.

\paragraph{P3 and over-restriction.} P3 guarantees the system can never autonomously weaken its defenses, but it also means the Autopilot cannot correct over-restriction. Resolution requires human intervention; D-UCB1 provides a slower, P3-compatible path.

\paragraph{Theoretical completeness vs.\ operational parsimony.} A tension exists between the three-term decomposition of $\hat{B}(x)$---which is theoretically necessary because each term requires a statistically distinct instrument (information-theoretic divergence for $U(x)$, hypothesis testing for $SB(x)$, sequential pattern matching for $RG(x)$)---and the operational complexity of maintaining multiple independent statistical mechanisms across a multi-stage pipeline. A production deployment may legitimately consolidate the three $\hat{B}(x)$ estimators into a single weighted composite:
\begin{equation}
RG_{\text{unified}}(x) = w_1 \cdot \mathrm{Drift}(x) + w_2 \cdot \mathrm{Bias}(x) + w_3 \cdot \mathrm{Sequence}(x)
\label{eq:unified}
\end{equation}
provided the underlying statistical instruments remain distinct at the estimation layer. Such consolidation is an implementation optimization that does not affect the theoretical decomposition---the $\hat{B}(x)$ terms remain conceptually independent even when their outputs are combined into a single scalar for the decision rule. We regard the gap between theoretical completeness and operational minimalism as a feature, not a defect: the decomposition tells you \emph{what} to measure and \emph{why}; the implementation decides \emph{how} to package those measurements for a specific deployment context.

\paragraph{Connection to Aubin's framework.} Three pieces of this work now sit inside Aubin's viability theory. First, the \emph{exit function} $\tau_{\mathrm{exit}}$ (Definition~\ref{def:exit_function}) and the first-order grounding of $t^*$ (Proposition~\ref{prop:tstar_exit}): $t^*$ is the leading-order estimator of the $VI$-based scalar surrogate $\tau^{VI}_{\mathrm{exit}}$ along the mean viability trajectory and, since $\tau^{VI}_{\mathrm{exit}} \geq \tau_{\mathrm{exit}}$ by construction, a conservative estimator of Aubin's $V$-exit itself; the non-linearity diagnostic $\rho$ is the empirical detector of when that first-order truncation ceases to be valid. Second, the \emph{Autopilot as regulation map} $R_{A}(\theta)$ (§\ref{sec:autopilot_regulation}, eq.~\ref{eq:ra_def}): tighten-only admissibility with $\mathrm{STOP}$ as admissible-control-of-last-resort gives a concrete instance of Aubin's $R(x) \subset U(x)$, and Theorem~\ref{thm:viab_ra_nonempty} establishes that the regulated viability kernel $\mathrm{Viab}_{R_{A}}(V^{\dagger})$ is non-empty by construction. Third, P3 is recovered as Corollary~\ref{cor:p3_regulation} rather than imposed as an external axiom. The formalisation surfaces one residual caveat relative to the classical framework: the stronger claim $\mathrm{Viab}_{R_{A}}(V) \neq \emptyset$ (non-empty \emph{without} appeal to $\mathrm{STOP}$) requires a dominance hypothesis $\nu_{\max} \geq C \cdot \beta_{\max}$ on tightening velocity versus drift rate that we flag as operating assumption rather than prove; empirical verification on production traces is scoped as follow-up work. The remaining open direction is the concurrency formalisation of P3 flagged under ``Monotonicity under concurrency and partial failures'' below.

\paragraph{On contract-style frameworks and P3.} Design-by-contract formalizations for autonomous agents, such as Agent Behavioral Contracts~\cite{bhardwaj2026}, decompose obligations as $C = (P, I, G, R)$, where $R$ denotes autonomous recovery, which in threshold-based governance takes the form of threshold relaxation after a detected violation. The declarative $(P, I, G)$ fragment is complementary to our pipeline and would strengthen the specification story flagged under ``Monotonicity under concurrency and partial failures'' in this section. $R$ interpreted as autonomous threshold relaxation, however, is incompatible with P3 by construction: it is precisely the adversarial erosion attack identified in Proposition~\ref{prop:p3}---a sophisticated adversary who can drive the recovery mechanism to lift $\tau$ inflates $d_{C_1}$ artificially and manipulates the very $\hat{B}(x)$ estimate the system uses to govern itself. The two admissible error modes are structurally asymmetric, not tunable: over-restriction is costly but reversible (a legitimate operation is blocked, a human adjudicates, the decision is revised with accountability); under-restriction is cheap but irreversible (harm is consumed, data is exfiltrated, fraud settles). P3 encodes this asymmetry as an invariant of $\hat{B}(x)$ integrity, which is why it is classified as the integrity pillar rather than the computational one (``P3: the weakest theoretical link'' above). We therefore integrate contract-style specifications only by reinterpreting $R$ as escalation rather than relaxation: a failed contract triggers $\mathrm{STOP}$, structured logging, and human notification---never autonomous threshold modification. The two P3-compatible paths to reduced restriction remain available---human-in-the-loop adjustment (trusted, auditable) and D-UCB1 threshold learning (evidence-driven, slow, §\ref{sec:autopilot})---and neither reacts to a single assertion failure. This design choice preserves the integrity argument of Proposition~\ref{prop:p3} while absorbing the declarative ergonomics that motivate recovery-based frameworks in benign deployments.

\section{Conclusion}

This paper has made two contributions. The first is theoretical: the \textbf{Informational Viability Principle}, which establishes that governing an autonomous agent is equivalent to continuously estimating $\hat{B}(x) = U(x) + SB(x) + RG(x)$ and acting only when $S(x) - \hat{B}(x) \geq \varepsilon$. This principle provides a unified motivation for the three AVF properties: P1 is necessary because $\hat{B}(x)$ requires accumulated cross-request state; P2 is necessary because preventing violations in high-consequence domains requires monitoring $d\hat{B}/dt$; P3 is motivated by the need to preserve the integrity of $\hat{B}(x)$ estimates under adversarial observation.

The second contribution is constructive: RiskGate, a concrete system that satisfies all three AVF properties by decomposing $\hat{B}(x)$ into its three terms and implementing dedicated statistical estimators for each. A monotonic pipeline ($D_{\text{final}} = \bigwedge_{i \in \mathcal{D}} D_i$) and a closed-loop Autopilot maintain $\hat{\Phi}(x) \geq \varepsilon$ autonomously. The scalar $VI(t) \in [-1, +1]$ operationalizing $\hat{\Phi}(x)$ with $t^*$ prediction transforms governance from reactive to predictive.

The worked demonstration traces a structuring attack showing how no single $\hat{B}(x)$ term detects the attack in isolation while their composition blocks it at plan-level. Validation against the Agents of Chaos failure taxonomy and convergence with Agent Drift measurements (81.5\% drift reduction with combined strategies, reported by Rath~\cite{rath2026}) confirm that agent degradation is an empirically observed phenomenon demanding multi-term adaptive $\hat{B}(x)$ estimation.

The Informational Viability Principle also provides a principled evaluation criterion for future governance approaches: any proposed mechanism should be evaluated by which term of $\hat{B}(x)$ it improves and whether it maintains P3 compliance. This transforms the governance design space from an open-ended collection of heuristics into a structured research program.

\paragraph{Open directions.} (1)~Reproducible benchmarks---a public dataset of agent operation sequences with labeled $\hat{B}(x)$ component values. (2)~Multi-agent viability---extending $V_{\text{system}} = \bigcap_i V_i$ with inter-agent coupling terms in $RG(x)$. (3)~Formal adversarial bounds---game-theoretic analysis of the attacker-governance interaction under P3 constraints on $\hat{B}(x)$ estimation, specifically addressing the integrity gap identified in section~12.

\section*{Acknowledgments}

The author gratefully acknowledges Prof. Dr. Maurits Kaptein, Juan Pablo Garcia, Claudio Aldana, and Rodrigo Alonso for their valuable feedback and support.


\end{document}